\documentclass{article}
\usepackage{arxiv}
\usepackage{microtype}
\usepackage{graphicx}
\usepackage{subcaption}
\usepackage{booktabs}
\usepackage{hyperref}
\usepackage{wrapfig}

\usepackage{amsmath}
\usepackage{amssymb}
\usepackage{mathtools}
\usepackage{amsthm}
\usepackage{amsfonts}
\usepackage{algorithm, algpseudocode}

\usepackage{natbib}
\usepackage[capitalize,noabbrev]{cleveref}

\theoremstyle{plain}
\newtheorem{theorem}{Theorem}
\newtheorem{proposition}[theorem]{Proposition}
\newtheorem{lemma}[theorem]{Lemma}

\theoremstyle{definition}

\newtheorem{assumption}[theorem]{Assumption}
\newtheorem{remark}[theorem]{Remark}
\newtheorem{example}{Example}
\theoremstyle{remark}

\title{Generation of Geodesics with Actor-Critic Reinforcement Learning to Predict Midpoints}
\author {
    Kazumi Kasaura \\
    OMRON SINIC X Corporation \\
    \texttt{kazumi.kasaura@sinicx.com}
}
\date{}

\DeclareMathOperator*{\argmin}{arg\,min}

\begin{document}

\maketitle

\begin{abstract}
To find the shortest paths for all pairs on manifolds with infinitesimally defined metrics, we introduce a framework to generate them by predicting midpoints recursively. To learn midpoint prediction, we propose an actor-critic approach. We prove the soundness of our approach and show experimentally that the proposed method outperforms existing methods on several planning tasks, including path planning for agents with complex kinematics and motion planning for multi-degree-of-freedom robot arms.
\end{abstract}


\section{Introduction}
\begin{figure}[ht]
    \centering
    \includegraphics[width=0.8\columnwidth]{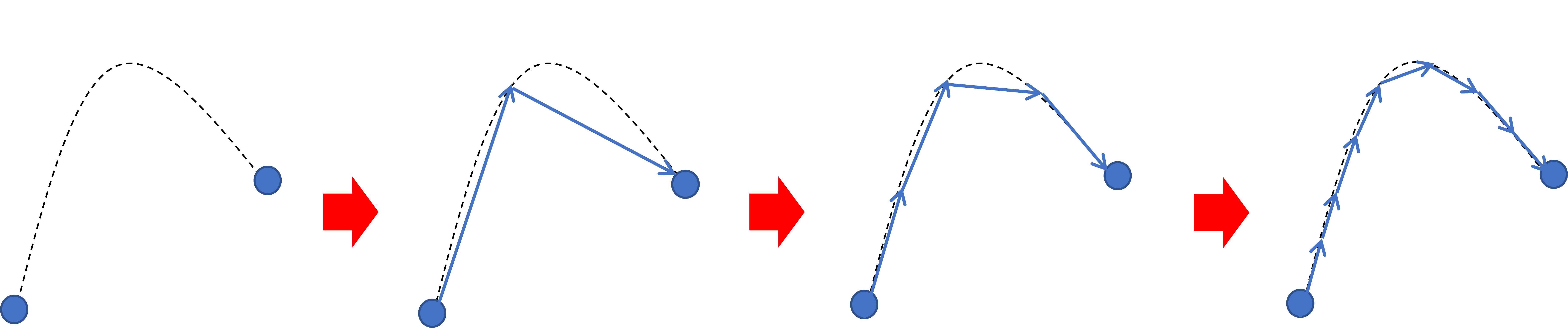}
    \caption{Midpoint tree generation of a geodesic (dotted curve).}
    \label{fig:teaser}
\end{figure}
On manifolds with metrics, minimizing geodesics, or shortest paths, are minimum-length curves connecting points.
Various real-world tasks can be reduced to the generation of geodesics on manifolds. 
Examples include time-optimal path planning on sloping ground~\citep{matsumoto1989slope}, robot motion planning under various constraints~\citep{lavalle2006planning,ratliff2015understanding}, physical systems~\citep{pfeifer2019finsler}, the Wasserstein distance~\citep{agueh2012finsler}, and image morphing~\citep{michelis2021linear,effland2021image}. Typically, metrics are only known infinitesimally (a form of a Riemannian or Finsler metric), and their distance functions are not known beforehand.
Computation of geodesics by solving optimization problems or  differential equations is generally computationally costly and requires an explicit form of the metric, or at least, values of its differentials.
To find geodesics for various pairs of endpoints in a fixed manifold, it is more efficient to use a policy that has learned to generate geodesics for arbitrary pairs.
In addition, we aim to learn only from infinitesimal values of the metric.

\begin{table}[ht]
    \centering
    \begin{tabular}{|r|ccc|}
    \hline
         & GCRL & SGT & \textbf{MT (Ours)}  \\
    \hline
         Prior Metric Knowledge& global approximation & global upper bound & local values\\
         Horizon & linear & logarithmic & logarithmic \\
         Step Length & fixed & variable & variable \\
    \hline
    \end{tabular}
    \caption{Comparison of Planning Frameworks}
    \label{tab:comparsion frameworks}
\end{table}

Goal-conditioned reinforcement learning (GCRL)~\citep{schaul2015universal} to generate waypoints of geodesics sequentially from start to goal has two issues. First, since it
suffers from sparseness of the reward if an agent gets a reward only when it reaches a goal, it is necessary to give the agent appropriate rewards when it gets near its goal. However, to define reward values, we must have a global approximation of the distance function between two points beforehand.
When manifolds have complex metrics, it may be difficult to find approximations.
Second, in trying to generate numerous waypoints, the long horizon makes learning difficult~\citep{wang2020long}.

Instead of generating waypoints in sequence, in the sub-goal tree (SGT) framework~\citep{jurgenson2020sub}, paths are generated by recursively applying a policy that have learned to predict intermediate points.
Since the recursion depth is the logarithm of the number of waypoints, the issue of long horizons is overcome.
This framework also has the advantage that parallelization allows for fast generation.
In addition, the density of waypoints can be varied by changing the depth of recurrence, whereas, in GCRL, the length of a single step is fixed in advance.
On the other hand, the issue of prior knowledge of the metric remains.
In this framework, the length of the "segment" between any two points must be given, which is an upper bound for the distance. However, it is not always easy to calculate them.
For example, in motion planning for a robotic arm, the segment between two states is not clearly defined, and verifying the validity of the corresponding motion is computationally expensive.

To overcome this difficulty, we propose a framework called \textit{midpoint tree} (MT), which is a modification of the sub-goal tree framework.
In this framework, a policy learns to predict midpoints of given pairs of points instead of arbitrary intermediate points, and paths are generated by inserting predicted midpoints recursively, as illustrated in \cref{fig:teaser}. Since adjacent pairs in the generated waypoints are close, the path length can be calculated even if the metric is known only locally, and the policy can be trained to generate shorter paths.

Table~\ref{tab:comparsion frameworks} summarizes the characteristics of the frameworks discussed in this introduction.

In addition to the aforementioned issue,
the original learning method for sub-goal prediction via a policy gradient in \citet{jurgenson2020sub} has poor sample efficiency when recursion is deep. To improve on this, we propose an actor-critic learning approach for midpoint prediction, which is similar to the actor-critic method~\citep{konda1999actor} for conventional reinforcement learning.

We prove theoretically that, under mild assumptions, if the training by our approach converges in the limit of infinite recursion depth, the resulting policy can generate true geodesics.
This result does not hold for generation by arbitrary intermediate points.

We experimentally compared our proposed method, on five path (or motion) planning tasks, to sequential generation with goal-conditioned reinforcement learning and midpoint tree generation trained by a policy gradient method without a critic. Two tasks involved continuous metrics or constraints (local planning), while the other three involved collision avoidance (global planning). In both local and global planning, our proposed method outperformed baseline methods for the difficult tasks.

\section{Related Works}
\subsection{Path Planning with Reinforcement Learning}
One of the most popular approaches for path planning via reinforcement learning is to use a Q-table~\citep{haghzad2017path,low2022modified}. However, that approach depends on the finiteness of the state spaces, and the computational costs grow with the sizes of those spaces.

Several studies have been conducted on path planning in continuous spaces via deep reinforcement learning~\citep{zhou2019learn,wang2019autonomous,kulathunga2022reinforcement,qi2022uav}.
In those works, the methods depend on custom rewards.

\subsection{Goal-Conditioned Reinforcement Learning and Sub-Goals}
Goal-conditioned reinforcement learning~\citep{kaelbling1993learning,schaul2015universal} trains a universal policy for various goals. It learns a value function whose inputs are both the current and goal states.
\citet{kaelbling1993learning} and \citet{dhiman2018floyd} pointed out that goal-conditioned value functions are related to the Floyd-Warshall algorithm for the all-pairs shortest-path problem~\citep{floyd1962algorithm}, as this function can be updated by finding intermediate states. They proposed methods that use brute force to search for intermediate states, which depend on the finiteness of the state spaces.
The idea of using sub-goals for reinforcement learning as options was suggested by \citet{sutton1999between}, and \citet{jurgenson2020sub} linked this notion to the aforementioned intermediate states.
\citet{wang2023optimal} drew attention to the quasi-metric structure of goal-conditioned value functions and suggested using quasi-metric learning~\citep{wang2022learning} to learn those functions.

The idea of generating paths by predicting sub-goals recursively has been proposed in three papers with different problem settings and methods. The problem setting for goal-conditioned hierarchical predictors~\citep{pertsch2020long} differs from ours because they use an approximate distance function learned from given training data, whereas no training data are given in our setting. Divide-and-conquer Monte Carlo tree search~\citep{parascandolo2020divide} is similar to our learning method because it trains both the policy prior and the approximate value function, which respectively correspond to the actor and critic in our method. However, their algorithm depends on the finiteness of the state spaces.

The problem setting for sub-goal tree framework~\citep{jurgenson2020sub} is the most similar to ours, but it remains different, as mentioned in the introduction.
In their main experiment on robot arm planning, local motion is executed by a controller, while only the global path is generated by a policy.
Consequently, the cost of the segment between two states is determined by the controller, and the recursion is not as deep as in our setting.

\section{Preliminaries and Notation}
In \S~\ref{subsec:quasi-metric space}, we describe general notions for quasi-metric spaces that are necessary to formulate our framework.
In \S~\ref{subsec:Finsler}, we describe Finsler manifolds, which serve as important examples of quasi-metric spaces and include formulations of various planning tasks.
On the other hand, formalizing our framework in a general form enables it to handle cases with hard collision avoidance constraints~(\S~\ref{subsec:obstacles}) that cannot be treated as Finsler manifolds.

\subsection{Quasi-Metric Space}\label{subsec:quasi-metric space}
We follow the notation in \citet{kim1968pseudo}.
Let $X$ be a space. A \textit{pseudo-quasi-metric} on $X$ is a function $d:X\times X\rightarrow \mathbb{R}$ such that
    $d(x,x)=0$, $d(x,y) \geq 0$, and
    $d(x,z)\leq d(x,y)+d(y,z)$
for any $x,y,z\in X$.
For $x \in X$ and $r>0$, the open ball $\{x'\in X|d(x,x')<r\}$ is denoted by $B_{r}(x)$.
A topology on $X$ is induced by $d$, which has the collection of all open balls as a base.
A pseudo-quasi-metric $d$ is called a \textit{quasi-metric}
if $d(x,y)>0$ for any $x,y\in X$ with $x\neq y$.

A pseudo-quasi-metric $d$ is called \textit{weakly symmetric} if $d(y_i,x) \to 0$
for any $x\in X$ and sequence $y_0, y_1, \ldots\in X$ with $d(x,y_i) \to 0$~\citep{arutyunov2017topological}.
When $d$ is weakly symmetric, $d$ is continuous as a function with respect to the topology it induces.

For two points $x,y \in X$, a point $z\in X$ is called a \textit{midpoint} between $x$ and $y$ if $d(x,z)=d(z,y)=d(x,y)/2$.
The space $(X,d)$ is said to have the \textit{midpoint property} if at least one midpoint exists for every pair of points.
The space $(X,d)$ is said to have the \textit{continuous midpoint property} if there exists a continuous map $m:X\times X\rightarrow X$ such that $m(x,y)$ is a midpoint between $x$ and $y$ for any $x,y\in X$~\citep{horvath2009note}.

Let $(X,d_X)$ and $(Y,d_Y)$ be pseudo-quasi-metric spaces. A pseudo-quasi-metric $d_{X\times Y}$ can be defined on $X\times Y$ by
\begin{equation}
    d_{X\times Y}((x_1,x_2),(y_1,y_2)):=d_X(x_1,y_1)+d_Y(x_2,y_2)
\end{equation}
and the induced topology coincides with the topology as a direct product. 

A function $f:(X, d_X) \to (Y, d_Y)$ is called \textit{uniformly continuous} if the following holds: For any $\varepsilon >0$, there exists $\delta>0$ such that, for any $x_1, x_2 \in X$ with $d_X(x_1, x_2) < \delta$, we have $d_Y(f(x_1), f(y_2)) < \varepsilon$.

A series $(f_0, f_1, \ldots)$ of functions from $(X, d_X)$ to $\mathbb{R}$ is called \textit{eventually equicontinuous} if the following holds~\citep{yang1969net}: For any $x\in X$ and $\varepsilon>0$, there exists $I \in \mathbb{N}$ and $\delta>0$ such that, for any $i \geq I$ and $x'\in B_{\delta}(x)$,
we have $\left|f_i(x)-f_i(x')\right| < \varepsilon$.
Note that each $f_i$ is not necessarily continuous.

\subsection{Finsler Geometry}\label{subsec:Finsler}
An important family of quasi-metric spaces is Finsler manifolds, including Riemannian manifolds as special cases.
Just like Riemannian manifolds, Finsler manifolds are also equipped with an infinitesimal metric, from which the distance is derived.
While Riemannian manifolds have symmetric metrics, Finsler manifolds allow for asymmetric ones, enabling the formulation of planning tasks with directional costs.

A \textit{Finsler manifold} is a differential manifold $M$ equipped with a function $F:TM\rightarrow [0,\infty)$, where $TM=\bigcup_{x\in M} T_xM$ is the tangent bundle of $M$, and $F$ satisfies the following conditions~\citep{bao2000introduction}.
\begin{enumerate}
\item $F$ is smooth on $TM\setminus 0$.
\item $F(x, \lambda v) = \lambda F(x, v)$ for all $\lambda>0$ and $(x,v) \in TM$.
\item The Hessian matrix of $F(x,-)^2$ is positive definite at every point of $TM_x\setminus 0$ for all $x\in M$.
\end{enumerate}

Let $\gamma:[0,1]\rightarrow M$ be a piecewise smooth curve on $M$. We define the length of $\gamma$ as
\begin{equation}
L(\gamma) := \int_0^1F\left(\gamma(t),\frac{d\gamma}{dt}(t)\right)dt.
\end{equation}

For two points $x,y\in M$, we define the distance $d(x,y)$ as
\begin{equation}
  d(x,y):=\inf\left\{L(\gamma)\middle|\gamma:[0,1]\rightarrow M,\gamma(0)=x, \gamma(1)=y \right\}.  
\end{equation}
Then, $d$ is a weakly symmetric quasi-metric~\citep{bao2000introduction}.
A curve $\gamma:[0,1]\rightarrow M$ is called a \textit{minimizing geodesic} if $L(\gamma)=d(\gamma(0),\gamma(1))$.

Clearly, $(M,d)$ has the midpoint property. It is known~\citep{bao2000introduction,amici2010convex} that, for any point of $M$, there exists a neighborhood $U\subseteq M$ such that any two points $p,q\in U$ can be uniquely connected by a minimizing geodesic inside $U$.
Note that $(U,d)$ has the continuous midpoint property.
That is, Finsler manifolds always have the continuous midpoint property locally (but not always globally).
See also Remark~\ref{rem:continuous_midpoint_property}.

\begin{example}\label{exam:matsumoto}
The \textit{Matsumoto metric} is an asymmetric Finsler metric that considers times to move on inclined planes~\citep{matsumoto1989slope}. Let $M \subseteq \mathbb{R}^2$ be a region on the plane with the standard coordinates $x,y$, and let $h:M\rightarrow \mathbb{R}$ be a differentiable function that indicates heights of the field. The Matsumoto metric $F:TM\rightarrow [0,\infty)$ is then defined as follows:
\begin{equation}
    F(p, (v_x,v_y)):=\frac{\alpha^2}{\alpha-\beta}
\end{equation}
where,
\begin{equation}
    \beta := v_x\frac{\partial h}{\partial x}(p)+v_y\frac{\partial h}{\partial y}(p),\,\alpha := \sqrt{v_x^2+v_y^2+\beta^2}. 
\end{equation}
This metric takes larger values on uphill slopes and smaller values on downhill slopes.
\end{example}

\begin{example}
For unidirectional car-like agents, the cost of trajectories considering kinematics as in \citet{rosmann2017kinodynamic} can be formulated as a Finsler metric. See \S~\ref{subsubsec:carlike} for the details.
\end{example}

\section{Theoretical Results}
We briefly explain our main idea to generate the shortest path between any two points in \S~\ref{subsec:midpoint tree}.
We prove propositions justifying our idea in \S~\ref{subsec:functional equation} and \S~\ref{subsec:iterative}. This result relies on the assumption that a function matching the true distance in the infinitesimal limit is known.
We present a construction of this assumed function for Finsler manifolds in \S~\ref{subsec:Finsler Case}.
In \S~\ref{subsec:obstacles}, we also show that we can handle the case where paths to be generated are restricted to a subset of a Finsler manifold, which enable our approach to be applied to tasks with hard collision avoidance constraints.

\subsection{Midpoint Tree}\label{subsec:midpoint tree}
We propose the \textit{midpoint tree} framework to solve the all-pairs shortest path problem in a given pseudo-quasi-metric space.
In this framework, a policy learns to predict a midpoint between any two points.
The shortest path between any start and goal points can be generated by recursively applying this prediction to adjacent pairs of previously generated waypoints, as illustrated in Figure~\ref{fig:teaser}.
We train the policy (actor) to predict midpoints by an actor-critic approach. In other words, we simultaneously train a function (critic) to predict distances.

In the following two subsections, we describe the theoretical considerations of this actor-critic learning approach.
In \S~\ref{subsec:functional equation}, we prove that, if certain functional equations hold and the critic coincide with the true distance in the infinitesimal limit, then the actor and the critic coincide with the true midpoint and distance, respectively.
This result is important as it justifies restricting actor predictions to midpoints~(Remark~\ref{rem:afterprop2}).
In \S~\ref{subsec:iterative}, we prove that these conditions hold if the actor and critic converge through iterative improvement.

\subsection{Functional Equation}\label{subsec:functional equation}

Let $(X, d)$ be a pseudo-quasi-metric space. When we do not know $d$ globally, we want to train a function (actor) $\pi : X \times X\rightarrow X$ to predict midpoints. 
We also train a function (critic) $V : X \times X\rightarrow \mathbb{R}$ to predict distances.

If $\pi$ and $V$ coincide with the true midpoints and distances, respectively, the following functional equations hold:
\begin{equation}\label{eq:Vpi}
    V(x,y)=V(x,\pi(x,y))+V(\pi(x,y),y),
\end{equation}
\begin{equation}\label{eq:piargmin}
    \pi(x,y)\in\argmin_z\left( V(x,z)^2+V(z,y)^2\right),
\end{equation}
\begin{equation}\label{eq:dpi}
    d(x,\pi(x,x))= d(\pi(x,x),x)=0.
\end{equation}
Note that $d(x,z)^2+d(z,y)^2$ takes its minimum value when $z$ is a midpoint between $x$ and $y$, because
\begin{equation}\label{eq:reason_midpoint}
d(x,z)^2+d(z,y)^2
= \frac{1}{2}(d(x,z)+d(z,y))^2+\frac{1}{2}(d(x,z)-d(z,y))^2.
\end{equation}
The first term takes the minimum value $d(x,y)^2/2$ when $z$ lies on some shortest path from $x$ to $y$ and the second term takes the minimum value $0$ when $z$ is equidistant from $x$ and $y$.

\begin{remark}\label{rem:reason_internal}
The above argument is generalized as follows: A point dividing $x$ and $y$ internally in the ratio $\alpha:1$ minimizes $d(x,-)^2 + \alpha d(-,y)^2$, because
\begin{equation}
d(x,z)^2+\alpha d(z,y)^2
= \frac{\alpha}{1+\alpha}(d(x,z)+d(z,y))^2+\frac{1}{1+\alpha}(d(x,z)-\alpha d(z,y))^2.
\end{equation}
\end{remark}

The following proposition states that, under mild assumptions, the conjunction of these functional equations and the condition that $V$ and $d$ are equal in the infinitesimal limit constitutes a sufficient (and obviously necessary) condition for $\pi$ and $V$ to coincide with the midpoints and distances, respectively.

\begin{proposition}\label{prop:secondlenma}
Assume that $(X,d)$ has the midpoint property and $\pi$ and $V$ satisfy (\ref{eq:Vpi}), (\ref{eq:piargmin}), and (\ref{eq:dpi}). Assume also that $\pi$ is uniformly continuous.

Let $\varepsilon \in (0, 1/9)$.
If there exists $\delta>0$ such that, $V$ approximates $d$ locally for pairs with distances less than $\delta$, which means that
\begin{equation}\label{eq:unifV}
     (1-\varepsilon)d(x,x')\leq V(x,x')\leq (1+\varepsilon)d(x,x') \text{ for any }x,x' \in X\text{ with }d(x,x')<\delta,
\end{equation}
then $V$ approximates $d$ globally, which means that
\begin{equation}
     (1-\varepsilon)d(x,y)\leq V(x,y)\leq (1+\varepsilon)d(x,y) \text{ for any }x,y \in X.
\end{equation}

In particular, if such $\delta>0$ exists for any $\varepsilon>0$, then $V=d$ and $\pi(x,y)$ is a midpoint between $x$ and $y$ for any $x,y\in X$.
\end{proposition}
\begin{proof}
Before proof, we explicitly rewrite the assumption of uniform continuity of $\pi$: For any $\alpha>0$, there exists $\beta>0$ such that, for any $x,y,z\in X$ with $d(y,z)<\beta$, we have
\begin{equation}\label{eq:unifpi}
     \max\{d(\pi(x,y),\pi(x,z)),d(\pi(y,x),\pi(z,x))\}<\alpha.
\end{equation}

We assume that $\delta>0$ satisfies the condition for a sufficiently small $\varepsilon>0$.

We first prove that $|V(x,y)|\leq (1+\varepsilon)d(x,y)$ for any $x,y\in X$.
We prove that $|V(x,y)|\leq (1+\varepsilon)d(x,y)$ when $d(x,y)<2^n\delta$ by induction for $n\in \mathbb{N}$.

For the case $n=0$, this is a direct consequence of (\ref{eq:unifV}) since $-(1+\varepsilon)d(x,y) \leq (1-\varepsilon)d(x,y)$.

We assume the induction hypothesis is true for $n$ and take $x,y\in X$ such that $d(x,y)<2^{n+1}\delta$. Let $m$ be the midpoint between $x$ and $y$.
Then,
\begin{equation}\label{eq:vdcalc2}
    \begin{split}
        V(x,y)^2
        &= \left(V(x,\pi(x,y))+V(\pi(x,y),y) \right)^2\\
        &\leq 2V(x,\pi(x,y))^2+2V(\pi(x,y),y)^2\\
        &\leq 2V(x,m)^2+2V(m,y)^2\\
        &\leq 2(1+\varepsilon)^2\left(d(x,m)^2+d(m,y)^2\right)\\
        &=(1+\varepsilon)^2d(x,y)^2
    \end{split}
\end{equation}
where the first equality comes from (\ref{eq:Vpi}),
the first inequality comes from $(a+b)^2\leq 2a^2+2b^2$,
the second inequality comes from (\ref{eq:piargmin}), and the third inequality comes from the induction hypothesis and $d(x,m)=d(m,y)<2^n\delta$.
Thus, $|V(x,y)|\leq (1+\varepsilon)d(x,y)$.

By induction, for all $x,y\in X$,
\begin{equation}\label{eq:above ineq}
|V(x,y)|\leq (1+\varepsilon)d(x,y).
\end{equation}

Next, we prove that $(1-\varepsilon)d(x,y)\leq V(x,y)$ for any $x,y\in X$ with the following strategy: To prove $(1-\varepsilon)d(x,y)\leq V(x,y)$, it is enough to show that this inequality holds for the pairs $(x,\pi(x,y))$ and $(\pi(x,y),y)$.
Thus, we use induction and apply the induction hypothesis to these pairs.
However, it is not straightforward to bound $d(x,\pi(x,y))$ and $d(\pi(x,y),y)$.
On the other hand, conversely, if the values of $V$ are close to those of $d$, $\pi$ is close to the midpoint.
Using this and the uniform continuity of $\pi$, we gradually extend the range in which the inequality holds, starting from the case where $\pi(x,y)$ is sufficiently close to both $x$ and $y$.

From the assumption of uniform continuity of $\pi$, we can take $\delta'>0$ such that, for any $x,y,z\in X$ with $d(y,z)< \delta'$,
\begin{equation}
     \max\{d(\pi(x,y),\pi(x,z)),d(\pi(y,x),\pi(z,x))\}<\delta.
\end{equation}
Especially, by considering the cases where $(x,y,z)$ is $(x,x,y)$ or $(y,x,y)$, when $d(x,y) < \delta'$,
\begin{equation}
     \max\{d(\pi(x,x),\pi(x,y)),d(\pi(x,y),\pi(y,y))\}<\delta.
\end{equation}
From (\ref{eq:dpi}) and the triangular inequality,
\begin{equation}
     \max\{d(x,\pi(x,y)),d(\pi(x,y),y)\}<\delta.
\end{equation}

Let $c_{\varepsilon}:=\left(1+2\sqrt{\varepsilon}+\varepsilon\right)/2$, which appears later in (\ref{eq:a_cl}). Since $\varepsilon < 1/9$, we have $(1-\varepsilon)^{-1}c_{\varepsilon}<1$. 
Using the uniform continuity of $\pi$ again, we can take $\eta'>0$ such that, for any $x,y,z\in X$ with $d(y,z)<\eta'$,
\begin{equation}
    \max\left\{d(\pi(x,y),\pi(x,z)), d(\pi(y,x),\pi(z,x))\right\}<\left(1-\frac{c_{\varepsilon}}{1-\varepsilon}\right)\delta'.
\end{equation}
Let $\eta := \min\{\eta'/2, \delta'\}$.

We prove by induction for $n\in\mathbb{N}$ that, when $d(x, y)<\delta' +n\eta$,
\begin{equation}\label{eq:IH21}
    (1-\varepsilon) d(x, y) \leq V(x, y),
\end{equation}
\begin{equation}\label{eq:IH22}
    (1-\varepsilon) d(x, \pi(x, y))\leq V(x, \pi(x, y)),
\end{equation}
\begin{equation}\label{eq:IH23}
    (1-\varepsilon) d(\pi(x, y),y)\leq V(\pi(x, y),y).
\end{equation}
Note that (\ref{eq:IH21}) follows from (\ref{eq:IH22}) and (\ref{eq:IH23}) because
\begin{equation}
\begin{split}
    (1-\varepsilon)d(x,y)&\leq (1-\varepsilon)(d(x,\pi(x,y))+d(\pi(x,y),y))\\
    &\leq V(x,\pi(x,y))+V(\pi(x,y),y)\\
    &=V(x,y).
\end{split}
\end{equation}

For the case $n=0$, if $d(x,y)<\delta'$, by the condition of $\delta'$, $d(x,\pi(x,y))$ and $d(\pi(x,y),y)$ are smaller than $\delta$. Thus, by the condition of $\delta$, (\ref{eq:IH22}) and (\ref{eq:IH23}) hold.

\begin{figure}[ht]
    \centering
    \includegraphics[width = 0.4\columnwidth]{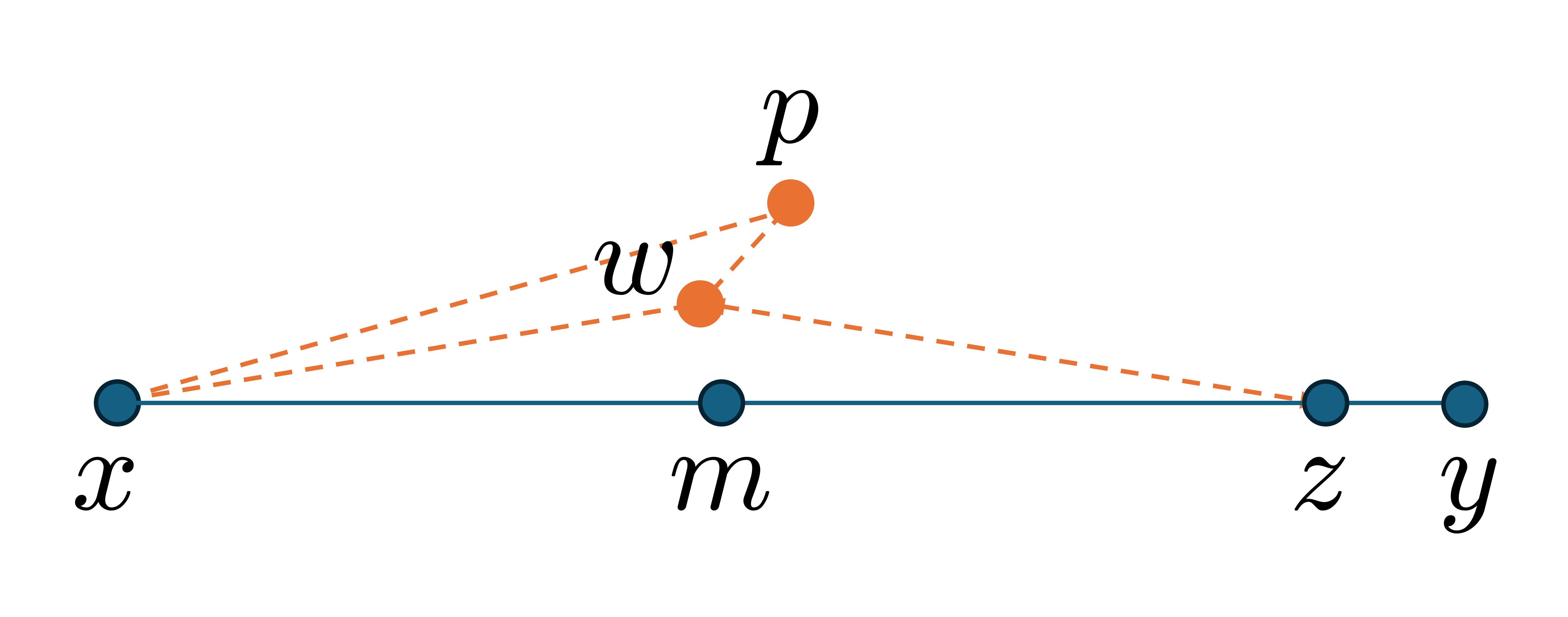}
    \caption{Positions of $x,y,z,w,m,p$.}
    \label{fig:proofgraph}
\end{figure}
We assume the induction hypothesis is true for $n$
and take $x,y$ such that $d(x,y) < \delta' +(n+1)\eta$.
If $d(x,y) < \delta' +n\eta$, the conclusions hold by the induction hypothesis. Thus, we can assume that $\delta' +n\eta \leq d(x,y)$.
By taking midpoints recursively, we can take $z\in X$ such that $d(x,z)+d(z,y)=d(x,y)$ and $\eta\leq d(z,y)<2\eta$ as shown in \cref{fig:proofgraph}. Let $w:=\pi(x,z)$, $a:=V(x,w)$, $b:=V(w,z)$, and $l:=d(x,z)$.
By (\ref{eq:Vpi}) and (\ref{eq:above ineq}), $ a+b = V(x,z)\leq (1+\varepsilon)l$.
On the other hand, as $l<\delta' +n\eta$, by (\ref{eq:IH21}) in the induction hypothesis, $a+b=V(x,z)\geq(1-\varepsilon)l$. Let $m$ be a midpoint between $x$ and $z$. Then, by (\ref{eq:piargmin}) and (\ref{eq:above ineq}),
\begin{equation}
a^2+b^2\leq V(x,m)^2+V(m,y)^2\leq (1+\varepsilon)^2\left(d(x,m)^2+d(m,z)^2\right)=\frac{(1+\varepsilon)^2l^2}{2}.
\end{equation}
Thus,
\begin{equation}\label{eq:a_cl}
    \begin{split}
        a&\leq \frac{1}{2}\left(a+b+|a-b|\right)\\
        &= \frac{1}{2}\left(a+b+\sqrt{2(a^2+b^2)-
        (a+b)^2}\right)\\
        &\leq \frac{1}{2}\left((1+\varepsilon)l+\sqrt{(1+\varepsilon)^2l^2-(1-\varepsilon)^2l^2}\right)\\
        &=c_{\varepsilon}l.
    \end{split}
\end{equation}
By (\ref{eq:IH22}) in the induction hypothesis, $d(x,w)\leq (1-\varepsilon)^{-1}a\leq (1-\varepsilon)^{-1}c_{\varepsilon}l$.
Let $p:=\pi(x,y)$. Then,
\begin{equation}
    d(x,p)\leq d(x,w)+d(w,p) < \frac{c_{\varepsilon}}{1-\varepsilon}l+\left(1-\frac{c_{\varepsilon}}{1-\varepsilon}\right)\delta'<\delta'+n\eta,
\end{equation}
where the second inequality comes from $d(z,y)<2\eta$ and the way $\eta$ is given.
Thus, by (\ref{eq:IH21}) in the induction hypothesis, $(1-\varepsilon)d(x,p)\leq V(x,p)$. By a symmetrical argument, we can also prove $(1-\varepsilon)d(p,y)\leq V(p,y)$.

Therefore, $(1-\varepsilon)d(x,y)\leq V(x,y)$ for any $x,y\in X$.

Thus, if the assumption holds for any $\varepsilon>0$, $V=d$. Then, by (\ref{eq:piargmin}) and the midpoint property, $\pi(x,y)$ is a midpoint between $x$ and $y$ for any $x,y\in X$.
\end{proof}

\begin{remark}\label{rem:afterprop2}
The conclusion does not follow if (\ref{eq:piargmin}) is replaced with
\begin{equation}\label{eq:piargmin2}
    \pi(x,y)\in\argmin_z \left(V(x,z)+V(z,y)\right).
\end{equation}
Let $f:[0,\infty)\rightarrow [0,\infty)$ be a non-decreasing subadditive function such that $\lim_{h\rightarrow +0} f(h)/h=1$ (for example, $f(h):=2\sqrt{1+h}-2$) and let $V:=f\circ d$. We consider the case $\pi(x,y)=x$.
Then, for any $x,y,z\in X$,
\begin{equation}
V(x,\pi(x,y))+V(\pi(x,y),y)=V(x,y) \leq f(d(x,z)+d(z,y))\leq V(x,z)+V(z,y).
\end{equation}
Thus, (\ref{eq:piargmin2}) and all the conditions of Proposition~\ref{prop:secondlenma} except (\ref{eq:piargmin}) are satisfied.
However, $V\neq d$ generally.
Note that $V\leq d$ follows even from these conditions,
which supports the following insights. If the upper bounds of distances are given, then the approach to predict arbitrary intermediate points can work as in \citet{jurgenson2020sub}. However, if distances can be approximated only for two close points, then it is necessary to avoid generating points near endpoints.
\end{remark}

\begin{remark}
In Proposition~\ref{prop:secondlenma}, it is necessary to assume that $\pi$ is uniformly continuous. Let $(X,d)$ be a pseudo-quasi-metric space with the midpoint property, and let $m(x,y)$ be the midpoint between $x,y\in X$. We consider the case where
\begin{equation}
V(x,y) = \begin{cases}
d(x,y) & d(x,y)\leq \sqrt{2}, \\
1 & \text{otherwise}
\end{cases}
\end{equation}
and 
\begin{equation}
\pi(x,y) = \begin{cases}
m(x,y) & d(x,y)\leq \sqrt{2}, \\
x & \text{otherwise}.
\end{cases}
\end{equation}
Then, for any $x,y,z \in X$,
\begin{equation}
\begin{split}
V(x,z)^2+V(z,y)^2 &\geq \min \left\{d(x,y)^2/2, 1\right\} \\
&= V(x,\pi(x,y))^2+V(\pi(x,y),y)^2.
\end{split}
\end{equation}
Thus, except for the assumption that $\pi$ is uniformly continuous, the conditions of Proposition~\ref{prop:secondlenma} are satisfied.
\end{remark}

\subsection{Iteration}\label{subsec:iterative}
Although we do not know the metric $d$ globally,
we assume that we know it infinitesimally, that is, we have a continuous function $C:X\times X\rightarrow \mathbb{R}$ that coincides with $d$ in the limit $d(x,y)\to 0$ or $C(x,y)\to 0$.
Formally, we assume the following conditions:
\begin{assumption}\label{assump:C}
\leavevmode
\makeatletter
\@nobreaktrue
\makeatother
\begin{enumerate}
    \item For $x\in X$ and a series $y_0, y_1, \ldots \in X$, if $C(x,y_i)\to 0$, then $d(x,y_i)\to 0$.
    \item For $x\in X$ and $\varepsilon>0$, there exists $\delta>0$ such that, for any $y,z\in B_{\delta}(x)$,
    \begin{equation}
         (1-\varepsilon)d(y,z)\leq C(y,z)\leq (1+\varepsilon)d(y,z).
    \end{equation}
\end{enumerate}
\end{assumption}

Starting with $C$, we consider iterative construction of actors and critics to satisfy (\ref{eq:Vpi}) and (\ref{eq:piargmin}). Formally, we assume series of functions, $V_i:X\times X \to \mathbb{R}$ and $\pi_i:X\times X \to X$, indexed by $i=0,1,\ldots$, that satisfy the following conditions:
\begin{equation}\label{eq:V0}
    V_0(x,y)=C(x,y),
\end{equation}
\begin{equation}\label{eq:defpii}
    \pi_i(x,y)\in\argmin_z \left(V_i(x,z)^2+V_i(z,y)^2\right),
\end{equation}
\begin{equation}\label{eq:Vi}
V_{i+1}(x,y)=V_i(x,\pi_i(x,y))+V_i(\pi_i(x,y),y).
\end{equation}
The following proposition states that, under mild assumptions, the limits of $\pi_0, \pi_1, \ldots$ and $V_0, V_1,\ldots $, if they exist, coincide with our desired outcome.

\begin{proposition}\label{prop:uniqueness}
Assume that $(X,d)$ is a compact, weakly symmetric pseudo-quasi-metric space with the midpoint property, and that series of functions $\pi_i:X\times X\rightarrow X$ and $V_i:X\times X\rightarrow \mathbb{R}$ satisfy (\ref{eq:V0}), (\ref{eq:defpii}), and (\ref{eq:Vi}). Assume also that the series $(V_0, V_1, \ldots)$ is eventually equicontinuous.
For functions $\pi:X\times X\rightarrow X$ and $V:X\times X\rightarrow \mathbb{R}$, we assume that $\pi_i(x,y)\to \pi(x,y)$ and $V_i(x,y)\to V(x,y)$ when $i\to\infty$ for any $x,y\in X$.
Then, if $\pi$ is continuous, $V$ and $\pi$ satisfy all conditions in Proposition~\ref{prop:secondlenma} and thus coincide with the exact distances and midpoints.
\end{proposition}
See \cref{appendix:proof of uniqueness} for the proof.
\begin{example}
We assume that $(X,d)$ has the midpoint property.
We consider the case in Remark~\ref{rem:afterprop2} where $f(h):=2\sqrt{1+h}-2$ and $C=f\circ d$. Then, $\pi_0$ coincide with the true midpoints. Indeed, by the triangle inequality and the monotonicity of $f$,
\begin{equation}
\begin{split}
C(x,z)^2+C(z,y)^2&\leq f\left(d(x,z)\right)^2+f\left(d(x,y)-d(x,z)\right)^2 \\
&= 16+4d(x,y)-8\left(\sqrt{1+d(x,z)}+\sqrt{1+d(x,y)-d(x,z)}\right),
\end{split}
\end{equation}
and the last term takes the maximum value when $d(x,z) = d(x,y)/2$. Thus, $C(x,z)^2+C(z,y)^2$ takes the minimum value when $z$ is the midpoint between $x$ and $y$.
Therefore,
\begin{equation}
V_1(x,y)=2f\left(\frac{d(x,y)}{2}\right)=4\sqrt{1+\frac{d(x,y)}{2}}-4.
\end{equation}
As before, $V_1(x,z)^2+V_1(z,y)^2$ takes the minimum value when $z$ is the midpoint between $x$ and $y$, and $\pi_1$ coincide with the true midpoints. Since this continues,
\begin{equation}
V_i(x,y)=2^{i+1}\sqrt{1+\frac{d(x,y)}{2^i}}-2^{i+1}
\end{equation}
and $\pi_n$ always coincide with the true midpoints.
By the Taylor expansion of the square root function at $1$, $V_i$ converges to $d$.
Note that, as shown in Remark~\ref{rem:afterprop2}, the iteration does not work if (\ref{eq:defpii}) is replaced by
\begin{equation}\label{eq:defpii_inter}
    \pi_i(x,y)\in\argmin_z \left(V_i(x,z)+V_i(z,y)\right).
\end{equation}
\end{example}

\begin{example}\label{example:cut}
We assume that $(X,d)$ has the midpoint property and the distance of any two points does not exceed a value $c \in \mathbb{R}$. Let $m(x,y)$ be the midpoint between $x$ and $y$. We consider the case where
\begin{equation}
C(x,y)=\begin{cases}
d(x,y) & d(x,y) \leq 1, \\
c & \text{otherwise}.
\end{cases}
\end{equation}
Then, if
\begin{equation}
\pi_0(x,y)=\begin{cases}
m(x,y) & d(x,y) \leq 2, \\
x & \text{otherwise},
\end{cases}
\end{equation}
(\ref{eq:defpii}) is satisfied. Thus,
\begin{equation}
V_1(x,y)=\begin{cases}
d(x,y) & d(x,y) \leq 2, \\
c & \text{otherwise}.
\end{cases}
\end{equation}
Since this continues,
\begin{equation}
V_i(x,y)=\begin{cases}
d(x,y) & d(x,y) \leq 2^i, \\
c & \text{otherwise}
\end{cases}
\end{equation}
and
\begin{equation}
\pi_i(x,y)=\begin{cases}
m(x,y) & d(x,y) \leq 2^{i+1}, \\
x & \text{otherwise}.
\end{cases}
\end{equation}
Therefore, $V_i$ converges to $d$ and $\pi_i$ converges to $m$.
Obviously, neither $V_i$ nor $\pi_i$ is continuous for small $i$.
Unlike the previous example, since $C$ give an upper bound of distance, this iteration works even if (\ref{eq:defpii}) is replaced by (\ref{eq:defpii_inter}). However, practically, this discontinuous $C$ is not suitable for learning. See the result of \textbf{Cut} in our experiments.
\end{example}

\subsection{Finsler Case}\label{subsec:Finsler Case}

We consider the case where $(X,d)$ is a Finsler manifold $(M,F)$. We assume that there exists a global coordinate system (a diffeomorphism to a subset) $f:M \hookrightarrow \mathbb{R}^d$. We consider the case where $C$ is defined as
\begin{equation}\label{eq:FinslerC}
    C(x,y):= F\left(x,df_x^{-1}(f(y)-f(x))\right),
\end{equation}
where $df_x:T_xM\rightarrow T_{f(x)}\mathbb{R}^d=\mathbb{R}^d$ is the differential of $f$ at $x$.

The following proposition states that Proposition~\ref{prop:uniqueness} is applicable to this case.

\begin{proposition}\label{prop:C}
$C$ satisfies Assumption~\ref{assump:C}.
\end{proposition}
See \cref{appendix:proof_of_C} for the proof.

\begin{remark}
The Assumption~\ref{assump:C} is preserved under weighted averaging. Thus, even if a Finsler manifold does not have a global coordinate system, we can construct $C$ satisfying Assumption~\ref{assump:C} by using local coordinate systems and a partition of unity.
Note that values of $C$ outside a neighborhood of the diagonal can be anything that does not converge to zero.
\end{remark}

\begin{remark}
When $M$ is compact, once the desired midpoint function $\pi:M\times M\rightarrow M$ is found, we can construct minimizing geodesics for all pairs of points. Let $A:=\left\{N/2^n \middle| n\geq0,\,0\leq N\leq 2^n\right\}$. For any $x,y\in M$, by applying $\pi$ recursively, we can construct $\gamma:A\rightarrow M$ such that $d(x,\gamma(a))=ad(x,y)$ and $d(\gamma(a),y)=(1-a)d(x,y)$ for any $a\in A$. For any $r\in [0,1]$, we can take a non-decreasing sequence $a_1,a_2,\ldots\in A$ such that $\lim_i a_i = r$. Then, because $\gamma(a_i)$ is a forward Cauchy sequence with respect to $d$, it converges~\citep{bao2000introduction}. Therefore, we can extend the domain of $\gamma$ to $[0,1]$.
\end{remark}

\begin{remark}\label{rem:continuous_midpoint_property}
Since we assume that $\pi$ is continuous in Proposition~\ref{prop:uniqueness}, it is appliable only for spaces with the continuous midpoint property.
In Finsler case, this property always holds true locally, as mentioned in \S~\ref{subsec:Finsler}.
On the other hand, our method is practically feasible even if this property does not hold.
Indeed, the environment in \S~\ref{subsubsec:matusmoto} does not have this property (Remark~\ref{rem:continuous_midpoint_property2}).
\end{remark}

\begin{remark}
    Instead of using (\ref{eq:FinslerC}), if we set
\begin{equation}\label{eq:AnotherFinslerC}
    C(x,y):= \int_0^1 F\left(l(t), df_{l(t)}^{-1}(f(y)-f(x))\right)dt,
\end{equation}
where
\begin{equation}
    l(t):=f^{-1}((1-t)f(x)+tf(y)),
\end{equation}
then $C$ gives the length of the curve connecting points as a segment in the coordinate space, which is an upper bound of the distance. In that case, we can use the setting of \citet{jurgenson2020sub}. However, the integration in (\ref{eq:AnotherFinslerC}) is not always efficiently computable.
\end{remark}
\begin{remark}
    Instead of using (\ref{eq:FinslerC}), we could define $C$ as
\begin{equation}
    C(x,y):= \frac{1}{2}\left(F\left(x,df_x^{-1}(f(y)-f(x))\right)+F\left(y,df_y^{-1}(f(y)-f(x))\right)\right).    
\end{equation}
Because this definition is not biased toward the $x$ side, it might be more appropriate for approximating distances.
The main reason we do not adopt this definition in this paper is the difficulty of using it in sequential reinforcement learning.
\end{remark}

\begin{remark}
For pseudo-Finsler manifolds, which are similar to Finsler manifolds but do not necessarily satisfy the condition of positive definiteness, the distance function $d$ can be defined and is a weakly symmetric quasi-metric~\citep{javaloyes2011definition}. However, we do not expect that Proposition~\ref{prop:C} generally holds for pseudo-Finsler manifolds.
\end{remark}

\subsection{Free Space}\label{subsec:obstacles}
Here, we consider the situation where there exist a \textit{free space} (a path-connected closed subset) $M_{\mathrm{free}}\subseteq M$, and we want to generate paths inside $M_{\mathrm{free}}$.
Examples include motion planning with obstacles in environments and multi-agent motion planning where collisions between agents are to be avoided.

We modify $d$ to
\begin{equation}
d'(x,y):=d(x,y)+P(x,y)
\end{equation}
where
\begin{equation}\label{eq:defP}
P(x,y):=\begin{cases}
    0 & x,y \in M_{\mathrm{free}}\text{ or }x,y \notin M_{\mathrm{free}},\\
    c_{P} & \text{otherwise},
\end{cases}
\end{equation}
and $c_{P}\gg 0$ is a constant.
Clearly, $d'$ is also a quasi-metric. When $M_{\mathrm{free}}$ is compact and $c_{P}$ is large enough, midpoints of any pair of points in $M_{\mathrm{free}}$ with respect to $d'$ lie in $M_{\mathrm{free}}$.

Let
\begin{equation}\label{eq:obsC}
C'(x,y):=C(x,y)+P(x,y).
\end{equation}
Obviously, $C'$ satisfies Assumption~\ref{assump:C} with respect to $d'$.

\begin{remark}
Whereas a procedure to determine whether direct edges cause collisions is assumed in \citet{jurgenson2020sub}, we use only a procedure to determine whether states are safe and aim to generate dense waypoints in the free space.
By introducing a margin between the free space and the obstacle region, one can ensure that the entire path lies outside the obstacle region, provided that consecutive waypoints are sufficiently close.
\end{remark}


\section{Experiments}\label{sec:experiments}
In \S~\ref{subsec:algorithm}, by building on the previous section's results,
we describe our proposed learning algorithm for midpoint prediction.
In the following subsections,
we compared our method, which generates geodesics via policies trained by our algorithm, with baseline methods on several path planning tasks. The first two environments deal with asymmetric metrics that change continuously (local planning tasks), while the other three environments have simple and symmetric metrics but involve obstacles (global planning tasks).

\subsection{Learning Algorithm}\label{subsec:algorithm}

Let $(X,d)$ be a pseudo-quasi-metric space, and let $C$ be a function satisfying Assumption~\ref{assump:C}.

We simultaneously train two networks with parameters $\theta$ and $\phi$: the \textit{actor} $\pi_{\theta}$, which predicts the midpoints between two given points, and the \textit{critic} $V_{\phi}$, which predicts the distances between two given points.
The critic learns to predict distances from the lengths of the sequences generated by the actor recursively predicting midpoints, where the length of a sequence is calculated as the sum of the values of $C$ for consecutive waypoints. On the other hand, the actor learns to predict midpoints from the critic's predictions.

We train only one actor and one critic, unlike in \S~\ref{subsec:iterative}. See also Remark~\ref{rem:single actor critic}.

The network for actor $\pi_{\theta}$ has a form for which the reparameterization trick~\citep{haarnoja2018soft} can be applied, that is, a sample is drawn by computing a deterministic function of the input, parameters $\theta$, and an independent noise. Henceforth, we abuse notations and denote a sampled prediction from $\pi_{\theta}(\cdot| s,g)$ by $\pi_{\theta}(s,g)$ even if it is not deterministic.

\begin{algorithm}[ht]
\caption{Actor-Critic Midpoint Learning}
\label{algorithm:Actor-Critic}
\begin{algorithmic}[1]
\State{Initialize $\theta$, $\phi$}
\While{learning is not done}\label{line:learningend}
\State{$\mathit{data}\gets\emptyset$}
\While{$\mathit{data}$ is not enough}\label{line:datacollectingend}
\State{$\mathit{data}\gets \mathit{data}\cup\mathrm{CollectData}(\pi_{\theta})$}
\EndWhile
\State{Split $\mathit{data}$ into $\mathit{batches}$}
\For{$\mathit{epoch}=1,\ldots,N_{\mathrm{epochs}}$}
\ForAll{$b\in \mathit{batches}$}
\State{Update $\phi$ with $\nabla_{\phi} \sum_{(s,g,c)\in b} L_{\mathrm{c}}(s,g,c)$, which is defined in (\ref{eq:Lc})}
\State{Update $\theta$ with $\nabla_{\theta} \sum_{(s,g,c)\in b} L_{\mathrm{a}}(s,g)$, which is defined in (\ref{eq:totalactorloss})}
\EndFor
\EndFor
\EndWhile
\State
\Procedure{CollectData}{$\pi$}
\State{Decide depth $D$}\label{line:decidedepth}
\State{Sample two points $p_{0},p_{2^D}$}\label{line:sampling}
\For{$i=0,\ldots,D-1$ and $j=0,\ldots,2^i-1$}
\State{$p_{2^{D-i-1}(2j+1)}\gets\pi_{\theta}(p_{2^{D-i}j},p_{2^{D-i}(j+1)})$}
\EndFor
\State{$\mathit{data}\gets\{(p_{0},p_{0},0),(p_{1},p_{1},0),\ldots,(p_{2^D},p_{2^D},0)\}$}
\State{$c_{D,0},c_{D,1}\ldots,c_{D,2^D-1}\gets C(p_{0},p_{1}),C(p_{1},p_{2}),\ldots,C(p_{2^D-1},p_{D,2^D})$}
\State{$\mathit{data}\gets \mathit{data}\cup\{(p_{0},p_{1},c_{D,0}),(p_{1},p_{2},c_{D,1}),\ldots,(p_{2^D-1},p_{2^D},c_{D,2^D-1})\}$}
\For{$i=D-1,\ldots,0$ and $j=0,\ldots,2^i-1$}
\State{$c_{i,j}\gets c_{i+1,2j}+c_{i+1,2j+1}$}
\State{$\mathit{data}\gets\mathit{data}\cup\{(p_{2^{D-i}j},p_{2^{D-i}(j+1)},c_{i,j})\}$}
\EndFor
\State\Return{$\mathit{data}$}
\EndProcedure
\end{algorithmic}
\end{algorithm}

\cref{algorithm:Actor-Critic} gives the pseudocode for our method.

We define the critic loss $L_{\mathrm{c}}$ for $s,g\in X$ with an estimated distance $c$ as
\begin{equation}\label{eq:Lc}
    L_{\mathrm{c}}:= L_{\mathrm{SLE}} + L_{\mathrm{symm}}^{\mathrm{c}},
\end{equation}
where
\begin{equation}
    L_{\mathrm{SLE}}:=\left(\log \left(V_{\phi}(s,g)+1\right)-\log \left(c+1\right)\right)^2.
\end{equation}
We take logarithms to reduce influence of large values.
If the distance $d$ is known to be symmetric, then setting the second term to
\begin{equation}
L_{\mathrm{symm}}^{\mathrm{c}}:=\left(\log \left(V_{\phi}(s,g)+1\right)-\log \left(V_{\phi}(g,s)+1\right)\right)^2
\end{equation}
ensures that $V_{\phi}$ is also symmetric. Otherwise, $L_{\mathrm{symm}}^{\mathrm{c}}:=0$.

The actor loss $L_{\mathrm{a}}$ for $s,g\in X$ is defined as
\begin{equation}\label{eq:totalactorloss}
L_{\mathrm{a}}:=L_{\mathrm{mid}}+L_{\mathrm{sm}}+L_{\mathrm{symm}}^{\mathrm{a}},
\end{equation}
where the first term,
\begin{equation}\label{eq:actorloss}
L_{\mathrm{mid}}:=V_{\phi}(s,\pi_{\theta}(s,g))^2+V(\pi_{\theta}(s,g),g)^2,
\end{equation}
is intended to make $\pi_{\theta}(s,g)$ a midpoint between $s$
 and $g$. The second term,
\begin{equation}
L_{\mathrm{sm}}:=V_{\phi}(\pi_{\theta}(s,g),\pi_{\theta}(\pi_{\theta}(s,\pi_{\theta}(s,g)),\pi_{\theta}(\pi_{\theta}(s,g),g)))^2
\end{equation}
ensures that the midpoint between the point of $1:3$ division and the point of $3:1$ division equals the midpoint between the original pair. This term aims to smooth the generated paths.
If the distance $d$ is known to be symmetric, then the third term
\begin{equation}
L_{\mathrm{symm}}^{\mathrm{a}}:=V_{\phi}(\pi_{\theta}(s,g),\pi_{\theta}(g,s))^2
\end{equation}
ensures that $\pi_{\theta}$ is also symmetric. Otherwise, $L_{\mathrm{symm}}^{\mathrm{a}}:=0$.

The data for training is collected using the actor $\pi_{\theta}$ with the current parameters.
We sample two points from $X$ and generate a sequence of points by repeatedly inserting points between adjacent points via $\pi_{\theta}$. Adjacent pairs of points at each iteration are collected as data.
The estimated distance between a collected pair of points is simply calculated as the sum of values of $C$ for adjacent pairs between the points in the final sequence.
In other words, we use a Monte Carlo method.
The number of recursion, which is denoted by $D$ and called \textit{depth},  is gradually increased during the learning process.

After collecting enough data, we update the parameters of the actor and critic according to the gradients of the sum of the aforementioned losses, via an optimization algorithm. We repeat this process of data collection and optimization a sufficient number of times.

\begin{remark}\label{rem:single actor critic}
For learning efficiency, we train the single actor and the single critic in this algorithm, 
while we theoretically consider series of actors and critics in \S~\ref{subsec:iterative}.
Instead, we gradually increase the depth for data collection to train the critic, so that $C$ is evaluated for closer pairs of points. In this way, $V$ and $\pi$ start by learning values of $V_0$ and $\pi_0$ and gradually learn the values of $V_i$ and $\pi_i$ for higher $i$. Note that, since $C$ is closer to the true distance $d$ for pairs of closer points, $V_i$ and $\pi_i$ converge more quickly in the region of pairs of closer points.
\end{remark}

\begin{remark}
While we use the Monte Carlo method to calculate the critic's target values, we could use TD($\lambda$)~\citep{sutton2018reinforcement} for $0\leq \lambda \leq 1$ instead, as $c_{D,j}:=C(p_{j},p_{j+1})$ and, for $i=D-1,\ldots,0$,
    \begin{equation}
    c_{i,j}:=
    (1-\lambda)(V_{\phi}(p_{2^{D-i}j},p_{2^{D-i-1}(2j+1)})
    +V_{\phi}(p_{2^{D-i-1}(2j+1)},p_{2^{D-i}(j+1)}))
    +\lambda(c_{i+1,2j}+c_{i+1,2j+1}).
    \end{equation}
\end{remark}

\subsection{Tasks and Evaluation Method}

We compared the methods in terms of their success rate on the following task.
A Finsler manifold $(M,F)$ with a global coordinate system $f:M\hookrightarrow \mathbb{R}^d$ (and a free space $M_{\mathrm{free}}\subseteq M$) is given as an environment.
The number $n$ of segments to approximate paths and a proximity threshold $\varepsilon>0$ are also fixed. For our method, $n$ must be a power of two: $n=2^{D_{\mathrm{max}}}$, where $D_{\mathrm{max}}$ is the midpoint trees' depth for evaluation.
When two points $s,g\in M_{\mathrm{free}}$ are given, we want to generate a sequence $s=p_0,p_1,\ldots,p_n=g$ of points such that no value of $C$ for two consecutive points is greater than $\varepsilon$, where $C$ is defined by (\ref{eq:FinslerC}) and, for cases with obstacles, (\ref{eq:obsC}).
If the points generated by a method satisfy this condition, then it is considered successful in solving the task; otherwise, it is considered to have failed.
Note that, for cases with obstacles, when $c_P>\varepsilon$, successes imply that all waypoints are lying on the free space.

For each environment, we randomly generated $100$ pairs of points from the free space, before the experiment. During training,
we evaluated the models regularly by solving the tasks for the prepared pairs and recorded the success rates.
We ran each method with different random seeds ten times for the environment described in \S~\ref{subsubsec:matusmoto} and five times for the other environments.

For each environment, the total number $T$ of timesteps was fixed for all methods.
Therefore, at Line~\ref{line:learningend} of our method, we continue learning until the number of timesteps reaches the defined value.
Timesteps were measured by counting the evaluation of $C$ during training. In other words, for sequential reinforcement learning (see the next subsection), the timesteps have their conventional meaning. For the other methods, one generation of a path (one cycle) with depth $D$ is counted as $2^D$ timesteps.

We mainly used success rate, not path length, for evaluation because metrics are only be defined locally and lengths thus cannot be calculated unless the success condition is satisfied. Furthermore, the evaluation by success rate allows for fair comparison with the baseline method using sequential generation.
However, we also compared lengths of generated paths in Appendix~\ref{appendix:length comparison}.

\subsection{Compared Methods}
The baseline methods were as follows:
\begin{itemize}
    \item \textbf{Sequential Reinforcement Learning (Seq)}: We formulated sequential generation of waypoints as a conventional, goal-conditioned reinforcement learning environment. The agent moves to the goal step by step in $M$.
    If the agent is at $p$, it can move to $q$ such that $F\left(p,df_p^{-1}(f(q)-f(p))\right)=\varepsilon$.
    If $q$ is outside the free space $M_{\mathrm{free}}$, the reward $R$ is set to $-c_{P}$ and the episode ends as a failure.
    Otherwise, $R$ is defined as
    \begin{equation}\label{eq:reward}
        R:= -\varepsilon + F\left(g,df_g^{-1}(f(g)-f(p))\right)
        -F\left(g,df_g^{-1}(f(g)-f(q))\right),
    \end{equation}
    where $g$ is the goal.\footnote{We do not define reward by using $C$ because $C(p,g)$ does not necessarily decrease when $p$ gets closer to $g$.}
    The discount factor is set to $1$.
    An episode ends and is considered a success when the agent reaches a point $p$ that satisfies $F\left(p,df_p^{-1}(f(g)-f(p))\right)<\varepsilon$.
    When the episode duration reaches $n$ steps without satisfying this condition, the episode ends and is considered a failure.
    
    We used Proximal Policy Optimization (PPO)~\citep{schulman2017proximal} to solve reinforcement learning problems with this formulation, which is a widely used sophisticated algorithm and has been used for path planning in recent studies~\citep{kulathunga2022reinforcement,qi2022uav}.
    
    \item \textbf{Policy Gradient (PG)}: We modified the method in \citet{jurgenson2020sub} to predict midpoints.
    For each $D=1,\ldots, D_{\mathrm{max}}$,
    a stochastic policy $\pi_D$ is trained to predict midpoints with depth $D$.
    To generate waypoints, we apply the policies in descending order of the index.
    We train the policies in ascending order of the index by the policy gradient.
    
    To predict midpoints, the value to be minimized in training is changed.
    Let $\rho(\pi_1, \ldots, \pi_D)$ be the distribution of $\tau:=(p_0,\ldots,p_{2^D})$, where $p_0$ and $p_{2^D}$ are sampled from the predefined distribution on $M$ and $p_{2^{i-1}(2j+1)}$ is sampled from the distribution $\pi_i\left(\cdot \middle| p_{2^{i}j}, p_{2^{i}(j+1)}\right)$.
    Let $\theta_D$ denote the parameters of $\pi_D$.
    Instead of minimizing the expected value of $\sum_{i=0}^{2^D-1}C(p_i,p_{i+1})$ as in the original method,
    we train $\pi_D$ to minimize the expected value of
    \begin{equation}\label{eq:PGloss}
        c_{\tau}:=\left(\sum_{i=0}^{2^{D-1}-1}C(p_i,p_{i+1})\right)^2+
        \left(\sum_{i=2^{D-1}}^{2^{D}-1}C(p_i,p_{i+1})\right)^2.
    \end{equation}
    Here, we use
    \begin{equation}\label{eq:PGtraining}
    \nabla_{\theta_D}\mathbb{E}_{\rho(\pi_1, \ldots, \pi_D)}\left[c_\tau\right]=
    \mathbb{E}_{\rho(\pi_1, \ldots, \pi_D)}\left[\left(c_{\tau}-b\left(p_0,p_{2^D}\right)\right)\nabla_{\theta_D}\log\pi_D\left(p_{2^{D-1}}\middle|p_0,p_{2^D}\right)\right],
    \end{equation}
    where $b$ is a baseline function.
%
    
    Note that, when the model is evaluated during training, if the current trained policy is $\pi_D$ ($1\leq D \leq D_{\mathrm{max}}$), then evaluation is performed with depth $D$ ($2^D$ segments). (The other methods are always evaluated with $n=2^{D_{\mathrm{max}}}$ segments.)
\end{itemize}

For our proposed method described in \S\ref{subsec:algorithm}, we tried two scheduling strategies to increase the trees' depth for training from zero to $D_{\mathrm{max}}$, under the condition of fixed total timesteps.

\begin{itemize}
\item \textbf{Timestep-Based Depth Scheduling (Our-T)}: For each depth, training lasts the same number of timesteps.
\item \textbf{Cycle-Based Depth Scheduling (Our-C)}: For each depth, training lasts the same number of calls to the data collection procedure (cycles).
\end{itemize}
More precisely, at Line~\ref{line:decidedepth} in \cref{algorithm:Actor-Critic}, 
for timestep-based depth scheduling,
the depth is $\lfloor t/t_d \rfloor$, where $t$ is the number of timesteps at the current time and $t_d:=\lfloor T/D_{\mathrm{max}}\rfloor+1$.
For cycle-based depth scheduling,
the depth for the $c$-th call to the data collection procedure is $\lfloor c/c_d\rfloor$, where $c_d:=\lfloor T/(2^{D_{\mathrm{max}}+1}-1)\rfloor+1$. 
Note that \textbf{Our-T} provides more training for low depths than \textbf{Our-C} does.

In addition, we ran the following variants of our method.

\begin{itemize}
\item \textbf{Intermediate Point (Inter)}: Instead of (\ref{eq:totalactorloss}), we use the following actor loss:
\begin{equation}
L_{\mathrm{a}}:=V_{\phi}(s,\pi_{\theta}(s,g))+V_{\phi}(\pi_{\theta}(s,g),g)+L_{\mathrm{symm}}^{\mathrm{a}},
\end{equation}
where, if $d$ is known to be symmetric, 
\begin{equation}
L_{\mathrm{symm}}^{\mathrm{a}}:= V_{\phi}(\pi_{\theta}(s,g),\pi_{\theta}(g,s)).
\end{equation}
Otherwise, $L_{\mathrm{symm}}^{\mathrm{a}}:=0$.
This means that $\pi_{\theta}$ learns to predict intermediate points that are not necessarily midpoints.
\item \textbf{2:1 Point (2:1)}: Instead of (\ref{eq:totalactorloss}), we use the following actor loss:
\begin{equation}
L_{\mathrm{a}}:=V_{\phi}(s,\pi_{\theta}(s,g))^2+2V_{\phi}(\pi_{\theta}(s,g),g)^2.
\end{equation}
By Remark~\ref{rem:reason_internal}, this means that $\pi_{\theta}$ learns to predict $2:1$ points instead of midpoints.

\item \textbf{Intermediate Point with Cut (Cut)}: To overcome the flaw of \textbf{Inter} in Remark~\ref{rem:afterprop2}, we use a modified version $C_{\varepsilon}$ of the function $C$ as
\begin{equation}
C_{\varepsilon}(x,y):=\begin{cases}
C(x,y) & \text{if } C(x,y)<\varepsilon,\\
c_{\mathrm{cut}} & \text{otherwise},
\end{cases}
\end{equation}
where $c_{\mathrm{cut}} \gg 0$ is a constant. We set $c_{\mathrm{cut}} := 30$.
\end{itemize}

For each environment, the depth scheduling method with better results in our proposed methods was adopted for these variants. Specifically,
we used cycle-based depth scheduling for the environments in \S~\ref{subsubsec:matusmoto} and \S~\ref{subsubsec:carlike}, and timestep-based depth scheduling for the environments in \S~\ref{subsubsection:2Dobstacles}, \S~\ref{subsubsection:robot}, and \S~\ref{subsec:treeagent}.

Note that \textbf{Seq} and \textbf{Cut} have a slight advantage because they use values of $\varepsilon$ while others do not.

\subsection{Environments}
We experimented in the following five environments.
\subsubsection{Matsumoto Metric}\label{subsubsec:matusmoto}

We consider the Matsumoto metric~(Example~\ref{exam:matsumoto}) where the region $M$ is the unit disk and the height function is $h(p):=-\|p\|^2$. Intuitively, this corresponds to a mountainous field with a peak at the center.

\begin{figure}[ht]
    \centering
    \includegraphics[width=0.25\columnwidth]{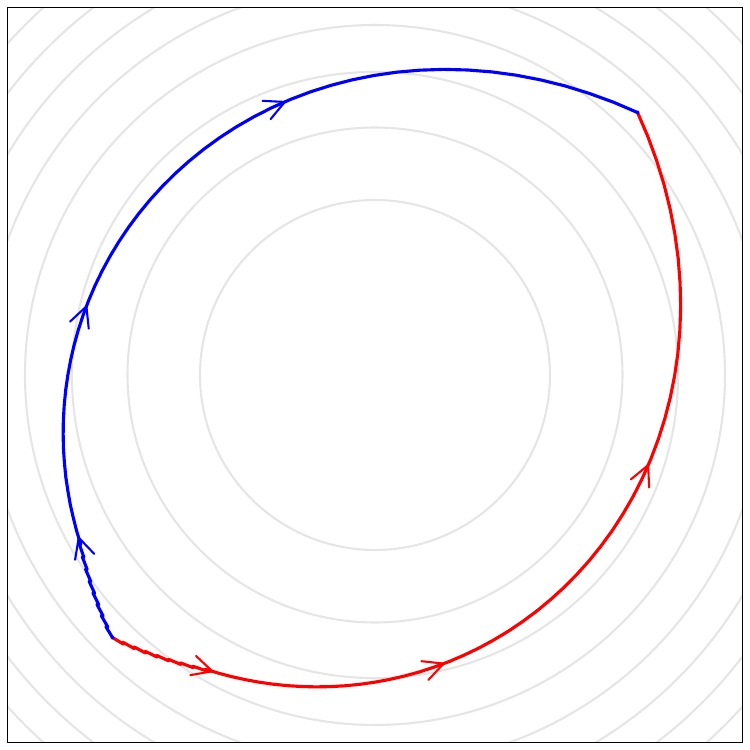}
    \caption{Example of switching geodesics.}
    \label{fig:matsumoto_example}
\end{figure}

\begin{remark}\label{rem:continuous_midpoint_property2}
This environment does not have the continuous midpoint property, because minimizing geodesics connecting two points on the opposite sides of the peak can switch between counterclockwise and clockwise routes, as illustrated in Figure~\ref{fig:matsumoto_example}.
\end{remark}

For this environment, 
we set the number of segments $n=64$ ($D_{\mathrm{max}}=6$), the proximity threshold $\varepsilon=0.1$, and the total number of timesteps $T=2\times10^7$.

\subsubsection{Unidirectional Car-Like Constraints}\label{subsubsec:carlike}
Inspired by the cost function for trajectories of car-like non-holonomic vehicles~\citep{rosmann2017kinodynamic}, we define an asymmetric Finsler metric for unidirectional car-like vehicles.

Let $M:=U\times S^1\subseteq \mathbb{R}^3$ be a configuration space for car-like vehicles, where $U\subseteq \mathbb{R}^2$ is a region with the standard coordinate system $x,y$ and $S^1$ is the unit circle with the coordinate $\theta$.
We define $F:TM\rightarrow [0,\infty)$ as follows:
\begin{equation}
F((p,\theta), (v_x,v_y,v_{\theta})):=\sqrt{v_x^2+v_y^2+c_{\mathrm{c}}(h^2+\xi^2)},
\end{equation}
where
\begin{equation}
h:=-v_x\sin(\theta)+v_y\cos(\theta),\,
\end{equation}
\begin{equation}\label{eq:xi}
\xi:=\max\left\{r_{\mathrm{min}}|v_{\theta}|-v_x\cos(\theta)-v_y\sin(\theta),0\right\},
\end{equation}
$c_{\mathrm{c}}$ is a penalty coefficient, and  $r_{\mathrm{min}}$ is a lower bound of radius of curvature.
The term $h$ penalizes moving sideways, while the term $\xi$ penalizes moving backward and turning sharply.

Strictly speaking, this metric is not a Finsler metric, as $F$ is not smooth on the entire $TM\setminus 0$.
However, we define the distance $d$ and function $C$ by using $F$ in the same way for a Finsler metric.
In the formula (\ref{eq:FinslerC}) for $C$, we take a value in $[-\pi,\pi]$ as the difference between two angles.

\begin{remark}
$F((p,\theta),-)^2$ is convex, and its Hessian matrix is positive definite where it is smooth.
In terms of \citet{fukuoka2021mollifier}, $F$ is a $C^0$-Finsler structure, which can be approximated by Finsler structures.
\end{remark}

\begin{remark}
While the vehicle model in \citet{rosmann2017kinodynamic} is bidirectional, that is, it can move backward, our model is unidirectional, that is, it can only move forward.
Unidirectionality seems essential for modeling by a Finsler metric. If we replace (\ref{eq:xi}) with
$\xi:=\max\left\{r_{\mathrm{min}}|v_{\theta}|-|v_x\cos(\theta)+v_y\sin(\theta)|,0\right\}$,
then $F$ cannot be approximated by Finsler structures because it does not satisfy subadditivity ($F(x,v+w)\leq F(x,v)+F(x,w)$).
\end{remark}

We take the unit disk as $U$ and set $c_{\mathrm{c}}=100$, $r_{\mathrm{min}}=0.2$, $n=64$ ($D_{\mathrm{max}}=6$), $\varepsilon=0.2$, and $T=8\times 10^7$.

\subsubsection{2D Domain with Obstacles}\label{subsubsection:2Dobstacles}
We consider a simple 2D domain with rectangular obstacles, as shown in \cref{fig:obstacle_trajectories}, where the white area is the free space, which is taken from an experiment in \citet{jurgenson2020sub}.
The metric is simply Euclidean.
Note that we also consider the outside boundary unsafe.

We set $n=64$ ($D_{\mathrm{max}}=6$), $\varepsilon=0.1$, $T=4\times 10^7$, and the collision penalty $c_P=10$.

\subsubsection{7-DoF Robotic Arm with an Obstacle}\label{subsubsection:robot}
This environment is defined for motion planning of the Franka Panda robotic arm, which has seven degrees of freedom (DoFs), in the presence of a wall obstacle.
The space is the robot's configuration space, whose axes correspond to joint angles.
The metric is simply Euclidean in the configuration space. The obstacle is defined as $\{x > 0.1,\,-0.1 < y < 0.1\}$, and a state is considered unsafe if at least one of the segments connecting adjacent joints intersects the obstacle. We ignored self-collision.

We set $n=64$ ($D_{\mathrm{max}}=6$), $\varepsilon=0.2$, $T=4\times 10^7$, and $c_P=10$.

\subsubsection{Three Agents in the Plane}\label{subsec:treeagent}
We consider three agents in $U:=[-1,1]\times[-1,1] \subseteq \mathbb{R}^2$.
The configuration space is $M:=U^3$ and the metric is defined as the sum of Euclidean metrics of three agents. A state is considered safe if all distances of pairs of agents are not smaller than $d_{\mathrm{thres}}:=0.5$, that is,
\begin{equation}
M_{\mathrm{free}}:=\left\{(p_0,p_1,p_2)\in M \middle|\, \|p_i-p_j\| \geq d_{\mathrm{thres}}\text{ for } i\neq j \right\}.
\end{equation}

We set $n=64$ ($D_{\mathrm{max}}=6$), $\varepsilon=0.2$, $T=8\times 10^7$, and $c_P=10$.

\subsection{Results and Discussion}\label{subsec:results}

\begin{figure*}[ht]
    \centering
    \includegraphics[width=\linewidth]{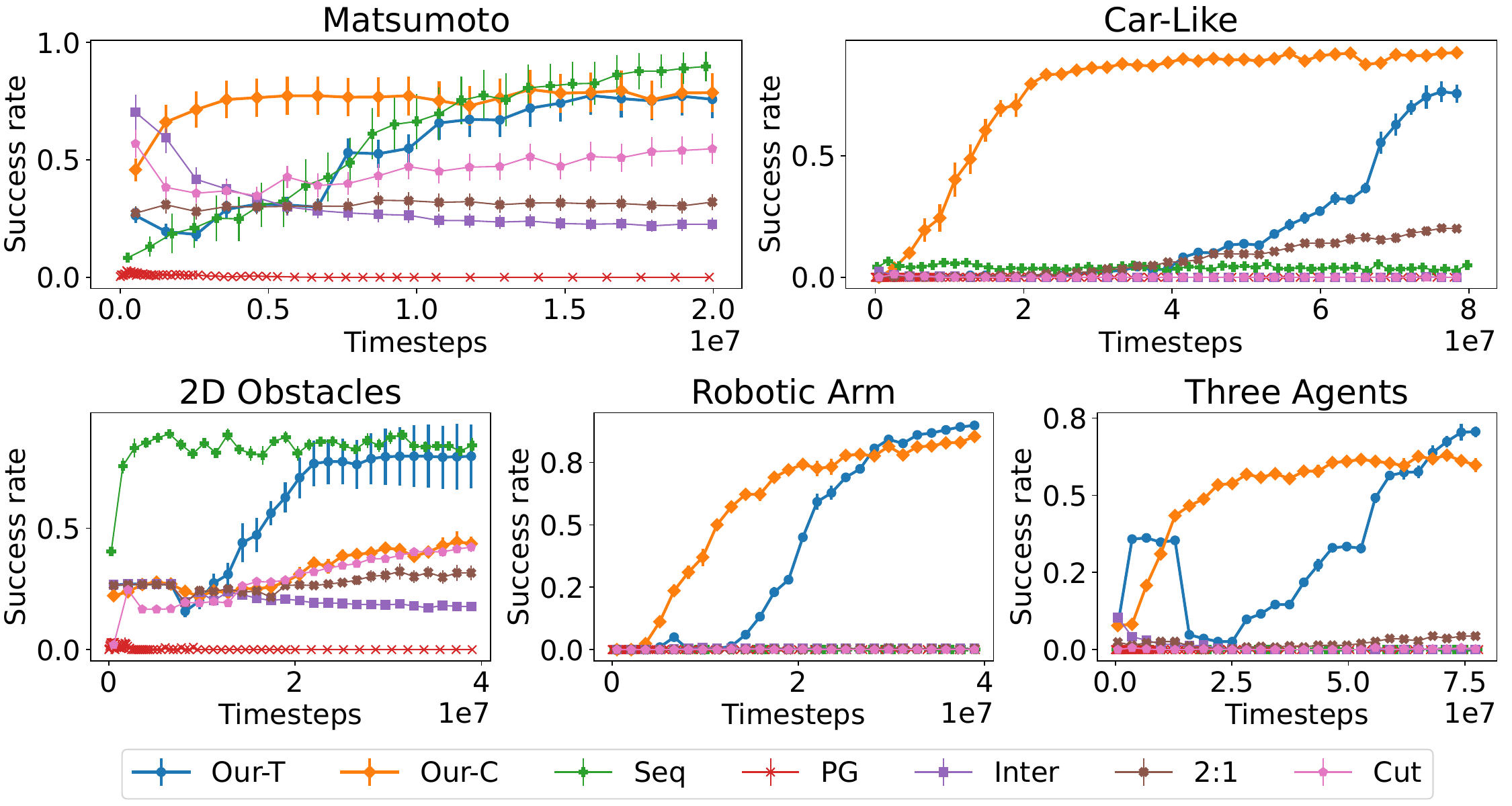}
    \caption{Success rate plots.}
    \label{fig:results}
\end{figure*}

\cref{fig:results} shows the learning curves of the success rates for all methods averaged over random seeds.
The error bars represent the standard errors.
While \textbf{Seq} achieved the best success rate in the Matsumoto and 2D obstacles environments, its success rates were low and our proposed method had the best in the other environments.
It may be much more difficult to determine directions of the agent sequentially in higher dimensional environments than in two-dimensional ones.
Also, the approximation of distances in (\ref{eq:reward}) may be close to the true values in the Matsumoto environment, but not in the car-like environment.
Note that, while it may be possible to improve learning success rate for \textbf{Seq} by engineering rewards, our method was successful without adjusting rewards.

The success rates for \textbf{Inter} decreased as training progressed in the Matsumoto and 2D obstacle environments, which may have resulted from convergence to biased generation, as mentioned in Remark~\ref{rem:afterprop2}.

While the success rates for \textbf{Cut} did not decrease as for \textbf{Inter}, it was less successful than our method and failed to learn in the environments with dimensions greater than two.
The discontinuity of prediction functions during training, as suggested by Example~\ref{example:cut}, may hinder learning.

The low success rates for \textbf{2:1} were, at least partially, due to its intentional uneven generation of waypoints.
On the other hand, it was the only method, aside from our proposed ones, that showed an improvement in the car-like environment, indicating that learning is possible even when generation points are fixed to a ratio other than 1:1. See also \cref{appendix:length comparison}.

\begin{figure}[ht]
\centering
\begin{subfigure}{0.32\columnwidth}
    \centering
    \includegraphics[width=1.0\columnwidth]{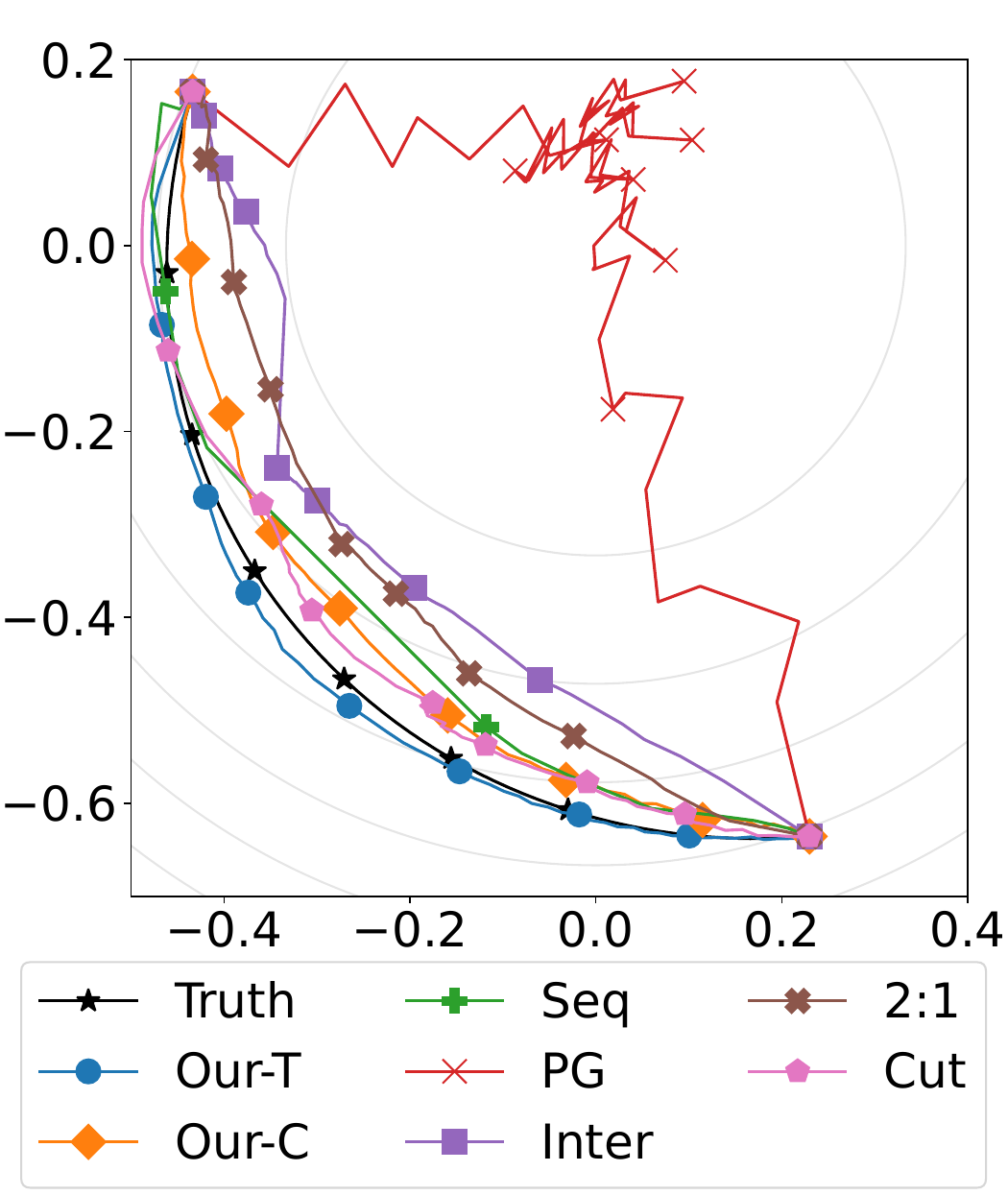}
    \caption{Matsumoto}
    \label{fig:trajectories}
\end{subfigure}
\begin{subfigure}{0.32\columnwidth}
    \centering
    \includegraphics[width=1.0\columnwidth]{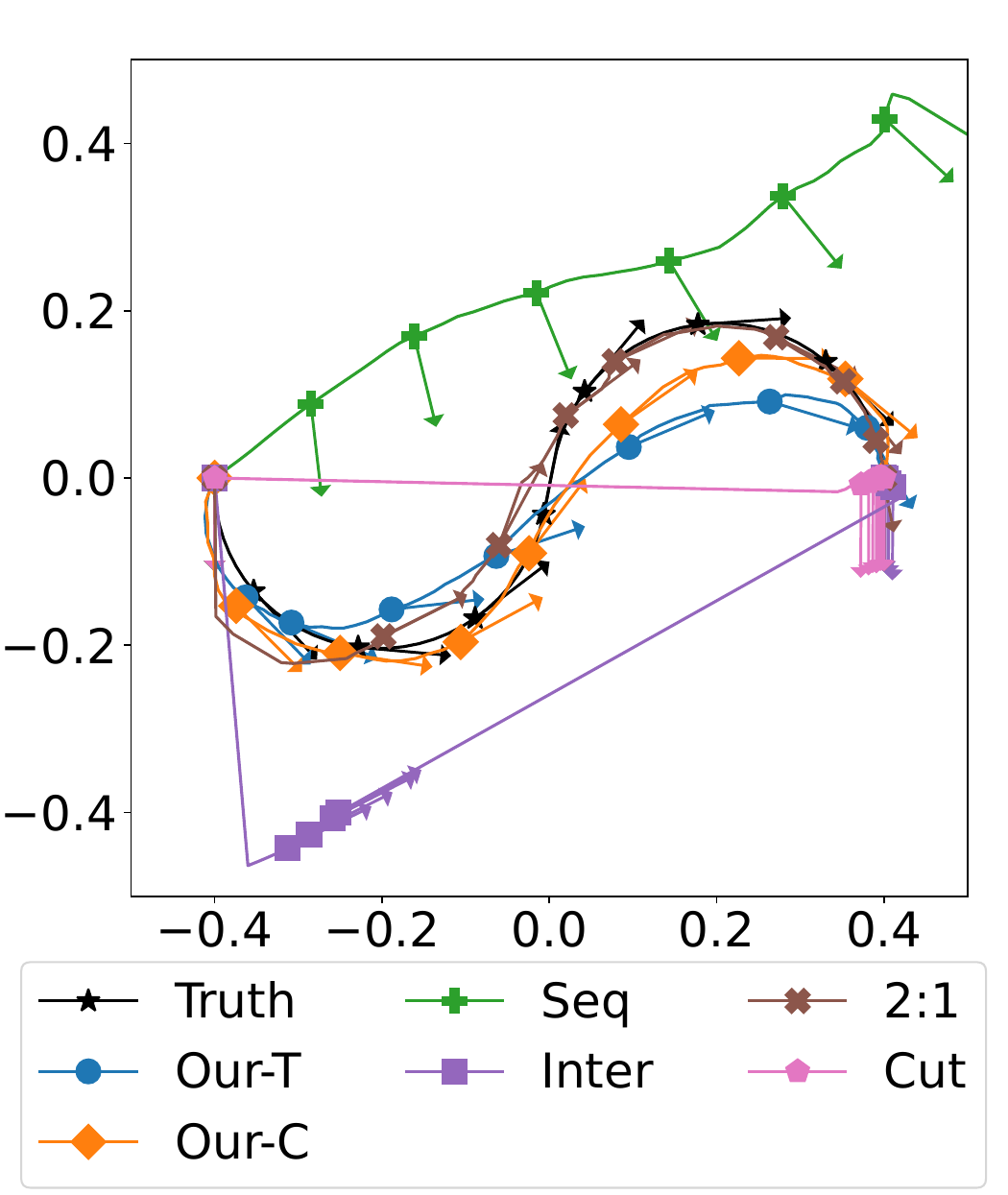}
    \caption{Car-Like}
    \label{fig:trajectories_car3}
\end{subfigure}
\begin{subfigure}{0.32\columnwidth}
    \centering
    \includegraphics[width=1.0\columnwidth]{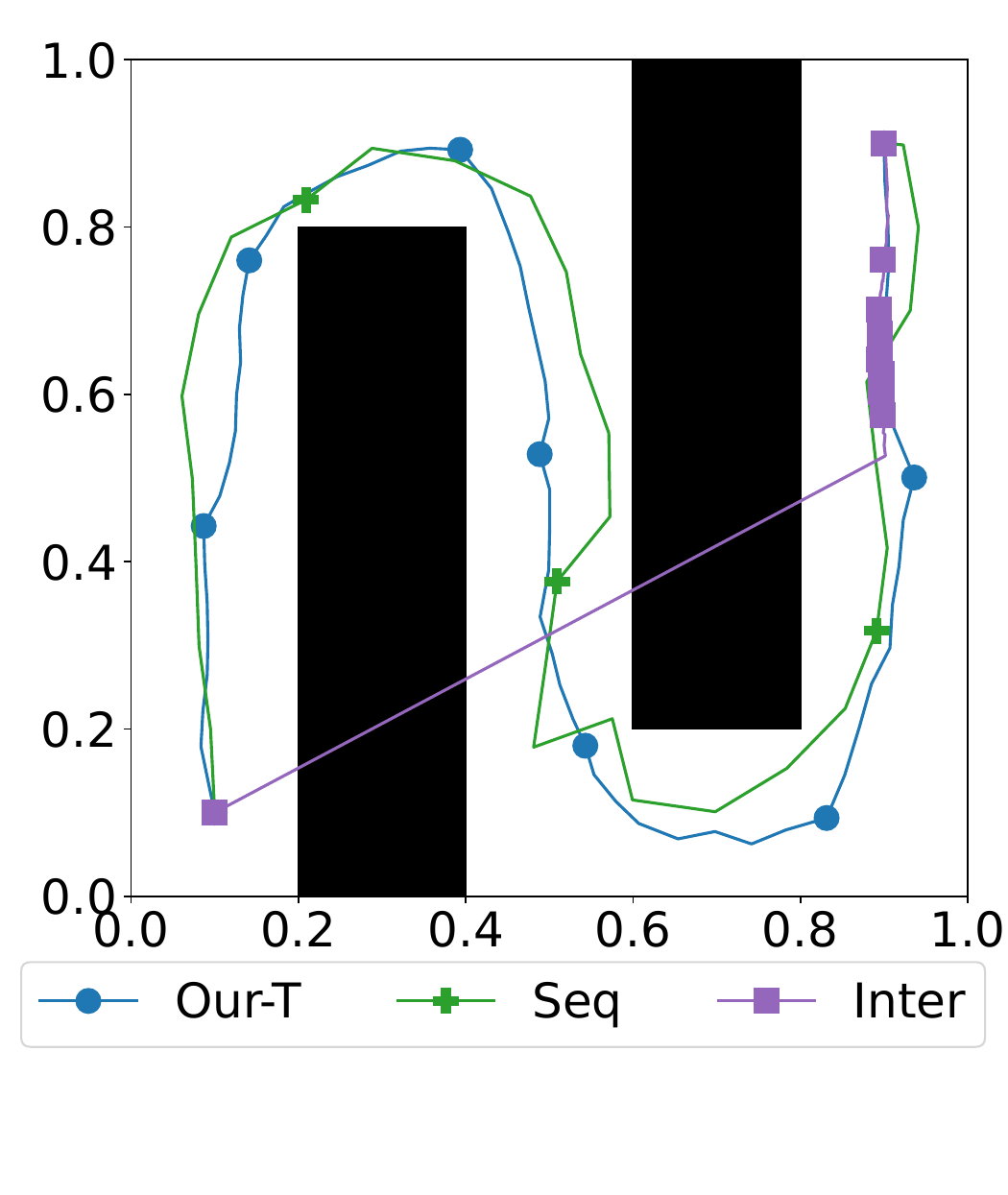}
    \caption{2D Obstacles}
    \label{fig:obstacle_trajectories}
\end{subfigure}
    \caption{Examples of generated paths.}
\end{figure}

\cref{fig:trajectories} shows examples of paths in the Matsumoto environment that were generated by policies trained by compared methods. Each eighth point is marked, and the circles represent contours. The \textbf{Truth} curve represents the ground truth of the minimizing geodesic with points dividing eight equal-length parts.
In this example, all methods except \textbf{PG} could generate curves close to the ground truth.
While \textbf{Our-T} and \textbf{Our-C} generated waypoints near the points dividing the true curve equally,
both \textbf{Inter} and \textbf{2:1} produced nonuniform waypoints.

\textbf{PG} was unable to solve most tasks in all environments and failed to even generate a smooth curve in the Matsumoto environment, which may have been simply due to insufficient training, or due to instability of learning without critic for deep trees.
Note that \textbf{PG} gets only one tuple of training data per path generation, whereas our method gets several, on the order of $2^D$ where $D$ is the depth, per path generation. 
Note also that, compared with the experiments in \citet{jurgenson2020sub}, trees were deeper and the success condition was stricter in ours.

\cref{fig:trajectories_car3} shows examples of paths generated by policies trained by methods except \textbf{PG} in the car-like environment. Orientations of the agent are represented by arrows. \textbf{Our-C} and \textbf{2:1} succeeded in generating paths close to the ground truth, while the waypoints for \textbf{2:1} were uneven. The path for \textbf{Our-T} was slightly straighter. Other methods failed to generate valid paths.

\cref{fig:obstacle_trajectories} shows examples of paths generated by policies trained by \textbf{Our-T}, \textbf{Seq}, and \textbf{Inter} in the 2D domain with obstacles. While the waypoints for \textbf{Our-T} and \textbf{Seq} avoided the obstacles, those for \textbf{Inter} skipped the obstacles. This is natural because, in \textbf{Inter}, it is not required to generate waypoints densely.

\begin{figure}[ht]
    \centering
    \includegraphics[width = \columnwidth]{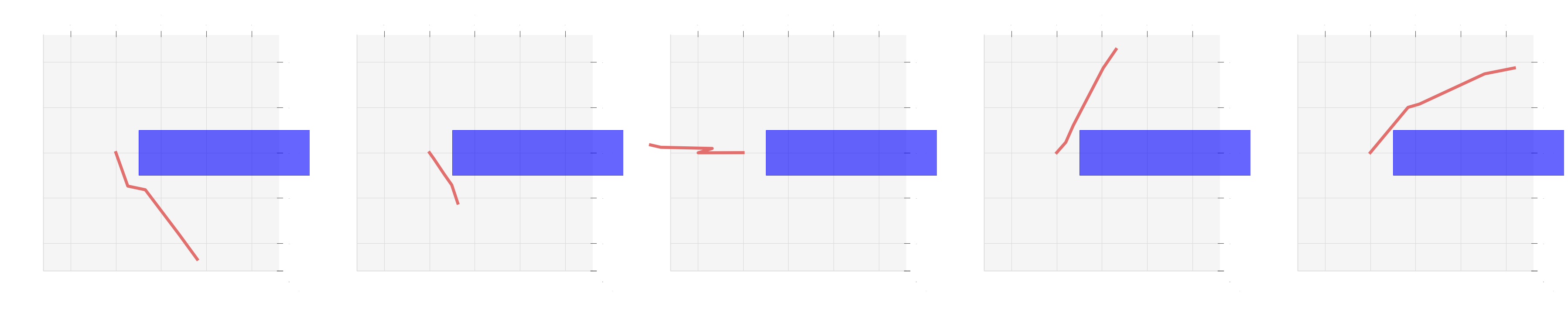}
    \caption{Generated motion for the robotic arm.}
    \label{fig:panda_motions}
\end{figure}

\cref{fig:panda_motions} shows a top view visualization of an example of a robotic arm motion generated by a policy learned by \textbf{Our-T}. The obstacle is illustrated in blue.
Despite the initial and final positions of the arm tips being on opposite sides of the obstacle, 
the motion succeeded in avoiding the obstacle.

\begin{figure}[ht]
    \centering
    \includegraphics[width = \columnwidth]{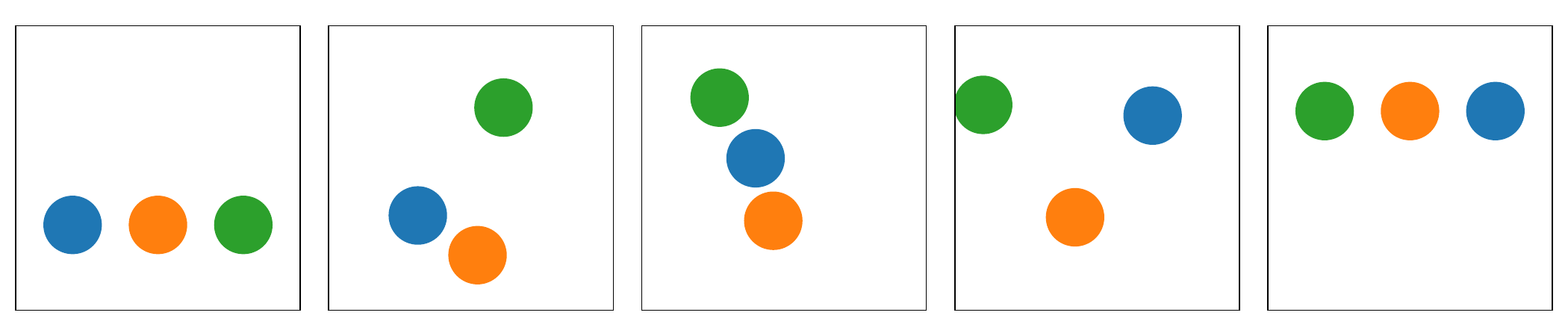}
    \caption{Generated motion for the three agents.}
    \label{fig:three_motions}
\end{figure}

\cref{fig:three_motions} shows an example of a motion in the three agents environments generated by a policy learned by \textbf{Our-T}. Agents succeeded to swap their positions while moving upward without collisions.

See \cref{appendix:length comparison} for comparisons of lengths of generated paths.

\section{Conclusion and Future Work}
In this paper, we proposed a framework, called midpoint trees, to generate geodesics by recursively predicting midpoints. We also proposed an actor-critic method for learning to predict midpoints, and we theoretically proved its soundness. Experimentally, we showed that our method can solve path planning tasks that existing reinforcement learning methods fail to solve.

While we assumed continuity of the policy function to prove our approach's theoretical soundness, the continuous midpoint property may only be satisfied locally.
Further research is needed on the conditions under which our method does not converge to wrong functions.
In addition, we were not able to discuss the conditions under which iterations converge in this paper, which is a topic for future work. 

Details of the learning algorithm in our proposed approach have not been fully investigated and may have room for improvement. For example, the followings could be considered.
\begin{itemize}
\item Our experiments showed that effective depth scheduling depends on the task. Exploration of an efficient depth scheduling algorithm that considers learning progress is a future challenge.
\item While we tried only a straightforward actor-critic learning algorithm in this paper,
algorithms for actor-critic reinforcement learning have been intensively researched.
Investigation of more efficient algorithms thus also remains for future work.
Especially,
while we used an on-policy algorithm, off-policy algorithms~\citep{degris2012off} may be useful in our framework.
\item While the architectures we used for both actors and critics were also simple,
the quasi-metric learning method~\citep{wang2022learning,wang2023optimal} may be useful for critics in our approach.
\end{itemize}

In our method, a policy must be learned for each environment.
By modifying our method so that the actor and critic input information on environments as in \citet{sartori2021cnn}, it may be possible to learn a policy that is applicable to different environments.

\section{Acknowledgements}
We thank the anonymous reviewers for their valuable comments and suggestions, which helped improve the quality of this paper.

\newpage

\appendix

\section{Proofs}\label{appendix:proof of lemmas}
We describe proofs omitted from the main text.

\subsection{Proof of Proposition~\ref{prop:uniqueness}}\label{appendix:proof of uniqueness}
We use the following lemma.

\begin{lemma}\label{prop:firstlemma}
Let $(X,d)$ be a pseudo-quasi-metric space.
Let $\pi_i:X\times 
 X \to X$ and $V_i:X\times X\to \mathbb{R}$ be two series of functions indexed by $i\in\mathbb{N}$ satisfying (\ref{eq:defpii}) and (\ref{eq:Vi}).
 Let $\pi:X\times 
 X \to X$ and $V:X\times X\to \mathbb{R}$ be two functions such that $\pi_i(x,y)\to \pi(x,y)$ and $V_{i}(x,y)\to V(x,y)$ when $i\to\infty$ for any $x,y\in X$.
 \begin{enumerate}
     \item\label{enum:1} Assume that the series of functions $(V_0, V_1,\ldots)$ is eventually equicontinuous. Then, for any $x,y\in X$,
\begin{equation}\label{eq:lemma13}
    \pi(x,y)\in\argmin_z \left(V(x,z)^2+V(z,y)^2\right),
\end{equation}
and
\begin{equation}\label{eq:lemma14}
V(x,y)=V(x,\pi(x,y))+V(\pi(x,y),y).
\end{equation}
\item\label{enum:2} Assume that $(X,d)$ is weakly symmetric and that $V_0(x,x)=0$, $V_0(x,y)\geq 0$, and $V_0(x,y)=0\implies d(x,y)=0$ for any $x,y\in X$. Then, for any $x\in X$,
\begin{equation}\label{eq:lemma15}
    d(x,\pi(x,x))=d(\pi(x,x),x)=0.
\end{equation}

\item\label{enum:3} For any $\varepsilon,\delta>0$, if for any $x,y\in X$,
\begin{equation}\label{eq:lemma16}
    V_0(x,y)<\delta\implies d(x,y)\leq(1+\varepsilon)V_0(x,y),
\end{equation}
then for any $x,y\in X$,
\begin{equation}\label{eq:lemma17}
    V(x,y)<\delta\implies d(x,y)\leq(1+\varepsilon)V(x,y),
\end{equation}
\item\label{enum:4} Assume that $(X,d)$ has the midpoint property. For any $\varepsilon,\delta>0$, if for any $x,y\in X$,
\begin{equation}\label{eq:lemma18}
    d(x,y)<\delta\implies V_0(x,y)\leq(1+\varepsilon)d(x,y),
\end{equation}
then for any $x,y\in X$,
\begin{equation}\label{eq:lemma19}
    d(x,y)<\delta\implies V(x,y)\leq(1+\varepsilon)d(x,y),
\end{equation}
\end{enumerate}

\end{lemma}
\begin{proof}
We first prove \ref{enum:1}. For any sequences $x_i$ and $y_i$ that converge $x$ and $y$, respectively,
\begin{equation}
    |V(x,y)-V_i(x_i,y_i)|\leq |V(x,y)-V_{i}(x,y)|
    +|V_{i}(x,y)-V_{i}(x_i,y_i)|.
\end{equation}
The first term can be bound by the convergence $V_i\to V$ at $(x,y)$ and the second term can be bound by the eventual equicontinuity of $(V_0, V_1, \ldots)$ at $(x,y)$ and convergence $x_i\to x$ and $y_i\to y$.
Thus, $\lim_{i\to\infty} V_i(x_i,y_i) = V(x,y)$. In particular,
\begin{equation}
    \lim_{i\to\infty}V_i(x,\pi_i(x,y))=V(x,\pi(x,y)),
\end{equation}
\begin{equation}
    \lim_{i\to\infty}V_i(\pi_i(x,y),y)=V(\pi(x,y),y).
\end{equation}

By (\ref{eq:defpii}), for any $x,y,z\in X$,
\begin{equation}V_i(x,\pi_i(x,y))^2+V_i(\pi_i(x,y),y)^2\leq V_i(x,z)^2+V_i(z,y)^2.
 \end{equation}
In the limit $i\to\infty,$
\begin{equation}V(x,\pi(x,y))^2+V(\pi(x,y),y)^2\leq V(x,z)^2+V(z,y)^2.
 \end{equation}
 Thus, (\ref{eq:lemma13}) follows. We can also show (\ref{eq:lemma14}) by taking the limit of both sides in (\ref{eq:Vi}).

Next, for \ref{enum:2},
when $V_i(x,x)=0$, because $V_i(x,\pi_i(x,x))=V_i(\pi_i(x,x),x)=0$ by (\ref{eq:defpii}), $V_{i+1}(x,x)=0$ by (\ref{eq:Vi}). Thus, $V_i(x,x)=V_i(x,\pi_i(x,x))=V_i(\pi_i(x,x),x)=0$ for all $i$ by induction.
As we can also show that $V_i(x,y)\geq 0$ for all $i$ by (\ref{eq:Vi}) and induction,
$V_{i+1}(x,y)=0$ implies $V_i(x,\pi_i(x,y))=V_i(\pi_i(x,y),y)=0$ by (\ref{eq:Vi}). Using this, $V_i(x,y)=0\implies d(x,y)=0$ for all $i$ is proven by induction. Therefore, $d(x,\pi_i(x,x))=d(\pi_i(x,x),x)=0$ for all $i$. Finally, by the assumption that $d$ is weakly symmetric, (\ref{eq:lemma15}) is proven.

As for \ref{enum:3}, because this assumption implies $V_0(x,y)\geq 0$, by (\ref{eq:Vi}) and induction, $V_i(x,y)\geq 0$ for all $i$.
Thus, $V_{i+1}(x,y)<\delta$ implies $V_i(x,\pi_i(x,y))< \delta$ and  $V_i(\pi_i(x,y),y)<\delta$ by (\ref{eq:Vi}).
Assume that $d(x,y)\leq(1+\varepsilon)V_i(x,y)$ for all $x,y\in X$ with $V_i(x,y)<\delta$. Then, when $V_{i+1}(x,y)<\delta$, because
\begin{equation}
d(x,\pi_i(x,y))\leq(1+\varepsilon)V_i(x,\pi_i(x,y))
\end{equation}
and 
\begin{equation}
d(\pi_i(x,y),y)\leq(1+\varepsilon)V_i(\pi_i(x,y),y),
\end{equation}
we can conclude that
\begin{equation}
d(x,y)\leq d(x,\pi_i(x,y))+d(\pi_i(x,y),y)
\leq (1+\varepsilon)V_{i+1}(x,y).    
\end{equation}
Therefore, $d(x,y)\leq(1+\varepsilon)V_i(x,y)$ when $V_i(x,y)<\delta$ for all $i$ by induction, which yields (\ref{eq:lemma17}).

Lastly, we prove \ref{enum:4}.
Assume that $d(x,y)<\delta\implies V_i(x,y)\leq(1+\varepsilon)d(x,y)$ for all $x,y\in X$.
Let $x,y\in X$ be a pair such that $d(x,y)<\delta$. By assumption, their midpoint $m$ exists. Then, by a similar calculation with (\ref{eq:vdcalc2}),
\begin{equation}\label{eq:vdcalc}
    \begin{split}
V_{i+1}(x,y)^2 &= \left(V_i(x,\pi_i(x,y))+V_i(\pi_i(x,y),y)\right)^2\\
&\leq 2V_i(x,\pi_i(x,y))^2+2V_i(\pi_i(x,y),y)^2\\
&\leq 2V_i(x,m)^2+2V_i(m,y)^2\\
&\leq 2(1+\varepsilon)^2\left(d(x,m)^2+d(m,y)^2\right)\\
&= (1+\varepsilon)^2d(x,y)^2,
    \end{split}
\end{equation}
where the first equality comes from (\ref{eq:Vi}), the second inequality comes from (\ref{eq:defpii}), and the third comes from the induction hypothesis and $d(x,m)=d(m,y)<\delta$. Thus, $d(x,y)<\delta\implies V_i(x,y)\leq(1+\varepsilon)d(x,y)$ for all $i$ by induction, which yields (\ref{eq:lemma19}).
\end{proof}

By the above lemma, it is sufficient to show the following to prove Proposition~\ref{prop:uniqueness}:
\begin{enumerate}
    \item $\pi$ is uniformly continuous.
    \item\label{enum:final_uniformity} For any $\varepsilon>0$, there exists $\delta>0$ such that the assumptions of \ref{enum:3} and \ref{enum:4} in Lemma~\ref{prop:firstlemma} are satisfied for $V_0=C$.
\end{enumerate}

To prove them, we use the following lemma, which is a generalization of a well-known fact to weakly symmetric pseudo-quasi-metric spaces.
\begin{lemma}\label{prop:thirdlemma}
Let $(X,d_X)$ be a compact pseudo-quasi-metric space and $(Y,d_Y)$ be a weakly symmetric pseudo-quasi-metric space.
If a function $f:X\rightarrow Y$ is continuous, then $f$ is uniformly continuous, that is, for any $\varepsilon>0$, there exists $\delta>0$ such that $d_Y\left(f(x),f(x')\right)<\varepsilon$ for any $x,x'\in X$ with $d_X(x,x')<\delta$.
\end{lemma}
\begin{proof}
We take $\varepsilon>0$ arbitrarily. Because $Y$ is weakly symmetric, for any $x\in X$,
we can take $\varepsilon(x)\leq\varepsilon/2$ such that for any $y\in Y$,
$d_Y(f(x),y)<\varepsilon(x)$ implies $d_Y(y,f(x))<\varepsilon/2$.
As $f$ is continuous, we can take $\delta(x)>0$ such that for any $x'\in X$,
$d_X(x,x')<2\delta(x)$ implies $d_Y(f(x),f(x'))<\varepsilon(x).$
Let $B(x):=B_{\delta(x)}(x)$. As $X$ is compact, we can take finite $x_1,\ldots,x_N$ such that $B(x_1),\ldots,B(x_N)$ covers $X$. Let $\delta:=\min_i \delta(x_i)$.

We take $x,x'\in X$ such that $d_X(x,x')<\delta$. We can then take $x_i$ such that $d_X(x_i,x)<\delta(x_i)$. As $d_Y(f(x_i),f(x))<\varepsilon(x_i)$ follows from this, $d_Y(f(x),f(x_i))<\varepsilon/2$. Because $d_X(x_i,x')<2\delta(x_i)$, $d_Y(f(x_i),f(x'))<\varepsilon(x_i)\leq\varepsilon/2$. Therefore, $d_Y(f(x),f(x'))<\varepsilon$.
\end{proof}

The uniform continuity of $\pi$ is a direct consequence of Lemma~\ref{prop:thirdlemma}.

Next,
we prove the existence of $\delta$ in the assumption of $\ref{enum:4}$ of Lemma~\ref{prop:firstlemma} by reducing it to uniform continuity of the ratio of $C$ to $d$.
We define a function $r:X\times X \rightarrow \mathbb{R}$ as
\begin{equation}
    r(x,y)=\begin{cases}
    \frac{C(x,y)}{d(x,y)} &d(x,y)\neq 0,\\
    1 &d(x,y)= 0.\\
    \end{cases}
\end{equation}
We show $r$ is continuous. We take $x,y\in X$ and series $x_0, x_1, \ldots\in X$ and $y_0, y_1, \ldots \in X$ that converge to $x$ and $y$, respectively. If $d(x,y)\neq 0$, because $d(x_i,y_i)\neq 0$ for sufficiently large $i$, $r(x_i,y_i)\to r(x,y)$ by the continuity of $C$ and $d$. Otherwise, because $x_i\to x$ and $y_i\to x$, $r(x_i,y_i)\to 1$ by Assumption~\ref{assump:C}. By Lemma~\ref{prop:thirdlemma}, $r$ is uniformly continuous, and the existence of $\delta$ in the assumption of $\ref{enum:4}$ of Lemma~\ref{prop:firstlemma} follows from this.
Note that $d(x,y)=0\iff C(x,y)=0$ from Assumption~\ref{assump:C}.

The existence of $\delta$ in the assumption of $\ref{enum:3}$ of Lemma~\ref{prop:firstlemma} follows from the uniform continuity of $r$ and the following lemma.

\begin{lemma}\label{prop:last lemma}
If $(X,d)$ is a compact, weakly symmetric pseudo-quasi-metric space and a continuous function $C:X\times X \rightarrow \mathbb{R}$ satisfies the first condition of Assumption~\ref{assump:C}, for any $\varepsilon>0$, there exists $\delta>0$ such that,  for any $x,y\in X$, $C(x,y)<\delta$ implies $d(x,y)<\varepsilon$.
\end{lemma}
\begin{proof}
Take arbitrary $\varepsilon>0$.
For any $x\in X$, we can take $\delta(x)>0$ such that for any $y\in X$, $C(x,y)<\delta(x)$ implies $d(x,y)<\varepsilon/2$. As $C$ is uniformly continuous by Lemma~\ref{prop:thirdlemma}, we can take $\eta(x)>0$ such that for any $y,z\in X$, $d(x,z)<\eta(x)$ implies $\left|C(x,y)-C(z,y)\right|<\delta(x)/2$. Because $X$ is weakly symmetric, after making $\eta(x)$ smaller if necessary, we can also assume that, for any $z\in X$, $d(x,z)<\eta(x)$ implies $d(z,x)<\varepsilon/2$. As $X$ is compact, we can take finite $x_1,\ldots,x_N$ such that for any $x\in X$, there exists $x_i$ such that $d(x_i,x)<\eta(x_i)$. Let $\delta:=\min_i \delta(x_i)/2$.

We prove that $C(x,y)<\delta$ implies $d(x,y)<\varepsilon$ for any $x,y\in X$. We can take $x_i$ such that $d(x_i,x)<\eta(x_i)$, which implies $d(x,x_i)<\varepsilon/2$ and $\left|C(x_i,y)-C(x,y)\right|<\delta(x_i)/2$. When $C(x,y)<\delta\leq \delta(x_i)/2$, because $C(x_i,y)<\delta(x_i)$, $d(x_i,y)<\varepsilon/2$. Thus, $d(x,y)<\varepsilon$.
\end{proof}

\subsection{Proof of Proposition~\ref{prop:C}}\label{appendix:proof_of_C}
To prove that the first condition of Assumption~\ref{assump:C} is satisfied, it is sufficient to show that $f(y_i)\to f(x)$ when $C(x,y_i)\to 0$, because the topology induced by $d$ coincides with the topology of the underlying manifold~\citep{bao2000introduction}. We can take the minimum value $c>0$ of $F\left(x,df_x^{-1}(\mathbf{v})\right)$ for $\mathbf{v} \in S^{d-1}:=\{\mathbf{v}\in\mathbb{R}^d|\|\mathbf{v}\|=1\}$ because $S^{d-1}$ is compact. Then, $C(x,y)\geq c\|f(y)-f(x)\|$ for any $y\in X$. Therefore, $f(y_i)\to f(x)$ when $C(x,y_i)\to 0$.

Next, we prove that the second condition of Assumption~\ref{assump:C} is satisfied. We fix $x \in X$ and $\varepsilon>0$. For a curve $\gamma:[0,1]\rightarrow M$, let
\begin{equation}
    L'(\gamma):=\int^1_0 F\left(x,df_x^{-1}\left(\frac{d\left(f\circ\gamma\right)}{dt}(t)\right)\right)dt,
\end{equation}
which is an approximation of $L(\gamma)$ and uses values of $F$ at $x$ instead of $\gamma(t)$.
Note that, although $\frac{d\left(f\circ\gamma\right)}{dt}(t)$ inherently belongs $T_{f(\gamma(t))}$, it can be naturally identified with $T_{f(x)}$ since there exists a natural basis in the tangent space of $\mathbb{R}^d$.
Let $E(\gamma)$ be the Euclidean length of $f\circ\gamma$.

We can take $\eta>0$ such that there exists $c'>0$ such that $L(\gamma)\geq c' E(\gamma)$ holds for any curve $\gamma$ contained in $B_{\eta}(x)$~\citep{busemann1941foundations}.
Let $\varepsilon':=\varepsilon c'$.

As $F$ is continuous and $S^{d-1}$ is compact, after making $\eta$ smaller if necessary, we can assume that, for any $y\in B_{\eta}(x)$ and $\mathbf{v}\in S^{d-1}$,
\begin{equation}
\left|F\left(y,df_y^{-1}(\mathbf{v})\right)-F\left(x,df_x^{-1}(\mathbf{v})\right)\right|\leq\varepsilon'.
\end{equation}
From this and the condition of $F$, for any $y\in B_{\eta}(x)$ and $\mathbf{v} \in \mathbb{R}^{d}$,
\begin{equation}
\left|F\left(y,df_y^{-1}(\mathbf{v})\right)-F\left(x,df_x^{-1}(\mathbf{v})\right)\right|\leq\varepsilon'\left\|\mathbf{v}\right\|.
\end{equation}
Then, for any curve $\gamma$ contained in $B_{\eta}(x)$,
\begin{equation}\label{eq:LL'}
\begin{split}
\left|L(\gamma)-L'(\gamma)\right|
&\leq\int_0^1\left|F\left(\gamma(t),\frac{d\gamma}{dt}(t)\right)-F\left(x,df_x^{-1}\left(\frac{d\left(f\circ\gamma\right)}{dt}(t)\right)\right)\right|dt\\
&\leq \varepsilon' \int_0^1 \left\|\frac{d\left(f\circ \gamma\right)}{dt}(t)\right\|dt\\
&= \varepsilon' E(\gamma).
\end{split}
\end{equation}

We can take $\delta$ such that for any $y,z\in B_{\delta}(x)$, there exists a minimizing geodesic from $y$ to $z$ contained in $B_{\eta}(x)$~\citep{bao2000introduction}, which is denoted by $\gamma(y,z)$.
Note that $d(y,z)=L(\gamma(y,z))$.
After making $\delta$ smaller if necessary, we can assume that $B_{\eta}(x)$ also contains the inverse image of the straight line segment from $f(y)$ to $f(z)$ by $f$, which is denoted by $l(y,z)$.
Note that $\|f(y)-f(z)\| = E(l(y,z))$.
Also note that and $C(y,z)=L'(l(y,z))$ because the derivative of $f\circ l(y,z)$ is the constant function to $f(z)-f(y)$.

As $L'$ is a Minkowskian metric, $L'(l(y,z))\leq L'(\gamma(y,z))$~\citep{busemann1941foundations}.
Therefore, for any $y,z\in B_{\delta}(x)$,
\begin{equation}
\begin{split}
d(y,z)-\varepsilon'\|f(y)-f(z)\|
&\leq L(l(y,z))-\varepsilon'\|f(y)-f(z)\|\\
&\leq C(y,z)\\
&\leq L'(\gamma(y,z))\\
&\leq d(y,z)+\varepsilon' E(\gamma(y,z)),
\end{split}
\end{equation}
where the first inequality comes from the fact that $\gamma(y,z)$ is the minimizing geodesic, the second comes from (\ref{eq:LL'}) with $\gamma = l(y,z)$, and the fourth comes from (\ref{eq:LL'}) with $\gamma = \gamma(y,z)$.
Because $\|f(y)-f(z)\|\leq E(\gamma(y,z))\leq d(y,z)/c'$, the condition is satisfied.

\section{Implementation Details}\label{appendix:ImpDetail}

\begin{table*}[ht]
\begin{center}
\begin{tabular}{llll}
& Ours & Seq & PG \\
\hline \\
Learning rate for \S~\ref{subsubsec:matusmoto}&$3\times10^{-5}$ & $3\times10^{-3}$& $5\times10^{-3}$\\
Learning rate for others &$10^{-6}$ & $3\times10^{-4}$& $5\times10^{-3}$\\
Batch size       &256&128&300\\
Number of epochs    &10 & 10 & 1 \\
Hidden layer sizes for \S~\ref{subsubsec:matusmoto} &[64, 64] &[64, 64] &[64, 64] \\
Hidden layer sizes for others &[400, 300, 300] &[400, 300, 300] &[400, 300, 300] \\
Activation function & ReLU & tanh & tanh \\
$\lambda$ for GAE  & - & 0.95 & - \\
Clipping parameter & - & 0.2 & 0.2 \\
Entropy coefficient & - & 0 & 1 \\
VF coefficient & - & 0.5 & - \\
Max gradient norm & - & 0.5 & - \\
Base standard derivation & - & - & 0.05 \\
Number of samples per episode & - & - & 10 \\
Number of episodes per cycle & - & - & 30 \\
\end{tabular}
\caption{Hyperparameter values.}
\label{tab:hyperparams}
\end{center}
\end{table*}

The codes used in the experiments for this paper are available at \url{https://github.com/omron-sinicx/midpoint_learning}.
\cref{tab:hyperparams} lists the hyperparameter values that we used for our method and the baseline methods.

The GPUs we used were NVIDIA RTX A5500 for experiments in \S~\ref{subsubsection:2Dobstacles} and \S~\ref{subsubsection:robot} and NVIDIA RTX A6000 for experiments in \S~\ref{subsubsec:matusmoto}, \S~\ref{subsubsec:carlike}, and \S~\ref{subsec:treeagent}.

\subsection{Environments}

The environments were implemented by NumPy and PyTorch~\citep{paszke2019pytorch}.
To implement of the robotic arm environment, we used PyTorch Kinematics~\citep{Zhong_PyTorch_Kinematics_2023}.
We also used Robotics Toolbox for Python~\citep{rtb} for visualization of the robotic arm motion.

For the coordinates in all environments except the angular component in the car-like environment (\S~\ref{subsubsec:carlike}), if a generated value is outside the valid range, it is clamped.
The angular component $S^1$ is represented by the unit circle $\{x^2+y^2=1\}$ in $\mathbb{R}^2$.
A point in $\mathbb{R}^2$ generated by a policy for this component is projected to $S^1$ by normalization after clamping to $[-1,1] \times [-1,1]$.
Note that, in this environment, the dimension of the space for state representation is different from the manifold dimension. 

\subsection{Our Method}\label{appendix:our method}

At Line~\ref{line:datacollectingend} in \cref{algorithm:Actor-Critic}, if the data size returned from one call to the data collection procedure is larger than the batch size, that procedure is called only once for one loop. Otherwise, it is called until one mini-batch is filled.
At Line~\ref{line:sampling}, two points are sampled from the uniform distribution over the free space.
We set the number of epochs $N_{\mathrm{epochs}}$ to $10$ and the batch size to $256$.

The actor network outputs a Gaussian distribution with a diagonal covariant matrix on the state representation space. During data collecting or training, a prediction by the actor is sampled from the distribution with the mean and deviation output by the network. During evaluating, the mean is returned as a prediction.
We use the reparameterization trick to train the actor as in SAC~\citep{haarnoja2018soft}.

Both the actor and critic networks
are multilayer perceptrons. The hidden layers were two of size $64$ for \S~\ref{subsubsec:matusmoto} and three of sizes $400, 300, 300$ for the other environments.
ReLU was selected as the activation function after trying tanh function, as well.
The output-layer size in the actor network is twice the dimension of the state representation space, where one half represents the mean and the other half represents logarithms of the standard deviations.
The critic network outputs a single value, whose exponential minus one is returned as a prediction of distance.
Adam~\citep{kingma2014adam} was used as the optimizer. The learning rate was tuned to $3\times 10^{-5}$ for \S~\ref{subsubsec:matusmoto} and to $10^{-6}$ for other environments.
PyTorch was used for implementation.

\subsection{Sequential Reinforcement Learning}\label{Appendix:ImpSeq}
In the conventional reinforcement learning environment, observations are pairs of current and goal states. Whenever an episode starts, start and goal points are sampled from the uniform distribution over the free space. The action space is $[-1,1]^d$, where $d$ is the manifold dimension. If the agent outputs $v$ for an observation $(f(p),f(g))$, the coordinate of the next state, $f(q)$, is calculated as
\begin{equation}\label{eq:move}
    f(q):= f(p)+\frac{\varepsilon}{F(p,f_p^{-1}(v))}v.
\end{equation}
If $f(q)$ is outside the coordinate space, it is clamped. Because $q$ cannot be calculated when exactly $v=0$, in such a case, $q$ is set to $p$ and the agent receives reward $R=-100$ as a penalty. Otherwise, the reward is calculated by (\ref{eq:reward}).
For the angular component in the car-like environment, the addition in (\ref{eq:move}) is calculated in $\mathbb{R}/2\pi\mathbb{Z}$.

We used PPO implemented in Stable Baselines3~\citep{raffin2019stable}, which uses PyTorch.
The discount factor was set to $1$.
The learning rate was tuned to $3\times10^{-3}$ for \S~\ref{subsubsec:matusmoto}, and to $3\times10^{-4}$ for the other environments. The batch size was tuned to $128$.
The network architectures were the same as those of the proposed method.
Other hyperparameters were set to the default values in the library.
The tanh function was selected as the activation function after trying ReLU, as well.

\subsection{Policy Gradient}\label{Appendix:ImpPG}
We modified the implementation of subgoal-tree policy gradient (SGT-PG)  by the authors, available at \url{https://github.com/tomjur/SGT-PG}, which uses TensorFlow~\citep{abadi2016tensorflow}.
The hyperparameter values except the hidden-layers sizes were the same as in the original paper.
We changed the network architectures to those of our proposed method.

The tanh function is used as the activation function.
The policy network outputs a Gaussian distribution with a diagonal covariant matrix on the state representation space.
The output-layers size is $2d$, where $d$ is the dimension of the state representation space. Let $m_1,\ldots, ,m_d,\sigma_1,\ldots,\sigma_d$ be the output for input $s,g$. The distribution mean is $(s+g)/2+(m_1,\ldots,m_d)^{\mathsf{T}}$, and the standard deviation for the $i$-th coordinate is $\mathrm{Softplus}(\sigma_i)+(0.05+\mathrm{Softplus}(c_i))\|s-g\|$, where $c_i$ is a learnable parameter.
While predictions are sampled from distributions during the policy training, we take the means as predictions during evaluation or training of other policies with higher indexes.

When we trained $\pi_D$, we sampled $30$ values of $(p_0,p_{2^D})$, the start and goal points, from the uniform distribution over the free space, in each training cycle.
For each sampled pair, we sampled $10$ values of $p_{2^{D-1}}$ from the distribution outputted by $\pi_D$, and we generate other waypoints deterministically by $\pi_{D-1},\ldots,\pi_1$ for each sampled value.
The average of $c_{\tau}$ was used as the baseline in (\ref{eq:PGtraining}).
The objective was that of PPO with an entropy coefficient of $1$ and clipping parameter of $0.2$. The optimizer was Adam, with the learning rate set to $5\times 10^{-3}$.

In the environments for \S~\ref{subsubsec:matusmoto}, we trained $\pi_1$ for $1000$ cycles and the other $\pi_D$ for $538$ cycles.
In the environment for \S~\ref{subsubsec:carlike} and \S~\ref{subsec:treeagent}, we trained each policy for $2117$ cycles.
In the environment for \S~\ref{subsubsection:2Dobstacles} and \S~\ref{subsubsection:robot}, we trained each policy for $1059$ cycles.
The total number of timesteps for \S~\ref{subsubsec:matusmoto} was $2\times 300\times 1000 + (4+8+16+32+64)\times 300\times 538=20613600\approx 2\times 10^7$, that for \S~\ref{subsubsec:carlike} and \S~\ref{subsec:treeagent} was $(2+4+8+16+32+64)\times 300\times 2117=80022600\approx 8\times 10^7$, and that for \S~\ref{subsubsection:2Dobstacles} and \S~\ref{subsubsection:robot} was $(2+4+8+16+32+64)\times 300\times 1059=40030200\approx 4\times 10^7$.

\section{Further Analysis of Experimental Results}\label{appendix:length comparison}
\begin{table}[ht]
    \centering
    \begin{subtable}[b]{0.99\textwidth}
    \centering
    \begin{tabular}{|c|ccccc|}
\hline
& Our-T& \textbf{Our-C}& Seq& Inter& 2:1\\
\hline
\textbf{Our-C}& $\mathbf{37 \pm 3\,(70)}$& & & & \\
Seq& $66 \pm 2\,(67)$& $\mathbf{73 \pm 3\,(70)}$& & & \\
Inter& $49 \pm 5\,(22)$& $\mathbf{57 \pm 7\,(22)}$& $62 \pm 8\,(20)$& & \\
2:1& $58 \pm 5\,(31)$& $\mathbf{70 \pm 4\,(32)}$& $63 \pm 5\,(29)$& $57 \pm 6\,(19)$& \\
Cut& $77 \pm 3\,(52)$& $\mathbf{85 \pm 2\,(53)}$& $69 \pm 5\,(49)$& $78 \pm 4\,(22)$& $74 \pm 4\,(30)$\\
\hline
\end{tabular}
    \caption{Matsumoto}
    \end{subtable}
    \begin{subtable}[b]{0.99\columnwidth}
    \centering
    \begin{tabular}{|c|cc|}
\hline
& Our-T& \textbf{Our-C}\\
\hline
\textbf{Our-C}& $\mathbf{18 \pm 3\,(77)}$& \\
2:1& $44 \pm 2\,(21)$& $\mathbf{73 \pm 3\,(21)}$\\
\hline
\end{tabular}
    \caption{Car-Like}
\end{subtable}
    \begin{subtable}[b]{0.99\columnwidth}
\vspace{10pt}
    \centering
    \begin{tabular}{|c|ccccc|}
\hline
& Our-T& Our-C& \textbf{Seq}& Inter& 2:1\\
\hline
Our-C& $73 \pm 10\,(39)$& & & & \\
\textbf{Seq}& $\mathbf{32 \pm 10\,(70)}$& $\mathbf{4 \pm 2\,(42)}$& & & \\
Inter& $42 \pm 15\,(16)$& $9 \pm 3\,(14)$& $\mathbf{77 \pm 5\,(16)}$& & \\
2:1& $52 \pm 15\,(30)$& $31 \pm 5\,(26)$& $\mathbf{84 \pm 5\,(30)}$& $70 \pm 8\,(15)$& \\
Cut& $69 \pm 10\,(36)$& $41 \pm 6\,(30)$& $\mathbf{89 \pm 2\,(38)}$& $80 \pm 6\,(16)$& $61 \pm 7\,(27)$\\
\hline
\end{tabular}
    \caption{2D Obstacles}
\end{subtable}
\begin{subtable}[b]{0.49\columnwidth}
    \centering
\begin{tabular}{|c|c|}
\hline
& \textbf{Our-T}\\
\hline
Our-C& $\mathbf{81 \pm 2\,(81)}$\\
\hline
\end{tabular}
\caption{Robotic Arm}
\end{subtable}
\begin{subtable}[b]{0.49\columnwidth}
    \centering
\begin{tabular}{|c|c|}
\hline
& Our-T\\
\hline
\textbf{Our-C}& $\mathbf{39 \pm 2\,(52)}$\\
\hline
\end{tabular}
\caption{Three Agents}
\end{subtable}
    \caption{Winning Rate Tables}
    \label{tab:winning_rates}
\end{table}

In addition to success rate,
we compared methods in the following way.
For each environment, we generated paths for all start and goal pairs in the evaluation data by the policies learned by all methods and, for all succeeded generations, calculated the path lengths, which are defined as the sum of $C$ values for consecutive waypoints.
For each pair of methods, in all start and goal pairs where both policies succeeded in generation, we calculated the percentage of pairs where the first policy generated a shorter path than the second policy.

\cref{tab:winning_rates} shows the percentages of instances where the methods in the top row outperform the methods in the left column. The values were averaged over random seeds and their standard errors are also displayed.
The numbers in parentheses represent the percentages of instances where both methods succeeded.
Methods with low success rates are omitted.
The methods which outperformed all other methods and their results are highlighted in bold.

While \textbf{Seq} had the highest success rate in the Matsumoto environment, \textbf{Our-T} and \textbf{Our-C} outperformed it in this comparison.
This may be because our methods generates denser waypoints and therefore smoother paths.
In contract, in the 2D obstacle environment, \textbf{Seq} outperformed \textbf{Our-T}, despite the small difference between their success rates.
This is probably because, in this environment, straight lines are the shortest where there are no obstacles.
While the success rates for \textbf{Our-T} are slightly higher than those for \textbf{Our-C} in both the robotic arm and three agents environments, \textbf{Our-T} outperformed \textbf{Our-C} in the former and vice versa in the latter.

It is interesting that \textbf{2:1} loses to our proposed method in this comparison, even though the non-uniformity of waypoints does not directly affect length.
A biased ratio in waypoints may make it difficult to generate smooth paths, or have a negative effect on learning.

\bibliography{ref}
\bibliographystyle{plainnat}

\end{document}